\documentclass{article}

\usepackage[preprint]{neurips_2026}

\usepackage[utf8]{inputenc} 
\usepackage[T1]{fontenc}    
\usepackage{hyperref}       
\usepackage{url}            
\usepackage{booktabs}       
\usepackage{amsfonts}       
\usepackage{nicefrac}       
\usepackage{microtype}      
\usepackage{xcolor}         

\usepackage{amsmath}
\usepackage{amssymb}
\usepackage{mathtools}
\usepackage{amsthm}
\usepackage{bbm}
\usepackage{amsfonts}
\usepackage{tabularx}
\usepackage{array}
\usepackage{capt-of}
\usepackage[textsize=tiny]{todonotes}
\usepackage{subcaption}
\usepackage{float} 
\usepackage[capitalize,noabbrev]{cleveref}

\theoremstyle{plain}
\newtheorem{theorem}{Theorem}[section]
\newtheorem{proposition}[theorem]{Proposition}
\newtheorem{lemma}[theorem]{Lemma}
\newtheorem{corollary}[theorem]{Corollary}
\theoremstyle{definition}
\newtheorem{definition}[theorem]{Definition}

\theoremstyle{remark}
\newtheorem{remark}[theorem]{Remark}
\newtheorem{example}[theorem]{Example}

\newcommand{\E}{\mathbb{E}}  
\newcommand{\R}{\mathbb{R}}
\DeclareMathOperator{\Tr}{Tr}

\title{A Basin-Selection Perspective on Grokking via Singular Learning Theory}

\author{%
  Ben Cullen\\
  Department of Computer Science\\
  University of Pisa\\
  Italy \\
  \And
  Sergio Estan-Ruiz\\
  Department of Mathematics\\
  Imperial College London\\
  United Kingdom \\
  \And
  Riya Danait\\
  Mathematical Institute\\
  University of Oxford\\
  United Kingdom \\
  \And
  Jiayi Li\\
  Section of Mathematics and Artificial Intelligence \\
  Max Planck Institute of Molecular Cell Biology and Genetics \\
  Center for Systems Biology Dresden \\
  Faculty of Mathematics\\ 
  TU Dresden \\
  Germany \\
}

\begin{document}
\maketitle
\begin{abstract}
  Grokking, the abrupt transition from memorization to generalisation after extended training, suggests the presence of competing solution basins with distinct statistical properties. We study this phenomenon through the lens of Singular Learning Theory (SLT), a Bayesian framework that characterizes the geometry of the loss landscape. The key measure is the local learning coefficient (LLC) which quantifies the local degeneracy of the loss surface. SLT links lower-LLC basins to higher posterior mass concentration and lower expected generalisation error. Leveraging SLT, we develop a basin-selection perspective on grokking in quadratic networks: LLC ranks competing near-zero-loss basins by statistical preference, while the training-time transition between them is governed by optimisation dynamics. In this view, grokking corresponds to a transition from a higher-LLC (memorising) basin to a lower-LLC (generalising) basin that dominates the posterior. To support this, we derive analytic formulas for the LLC in shallow quadratic networks under both lazy and feature learning regimes. Empirically, we demonstrate that LLC trajectories estimated from training data track the onset of generalisation and provide an informative probe of the optimisation path.
\end{abstract}

\section{Introduction}\label{sec::introduction}
Grokking refers to a training phenomenon in which a model attains near-zero empirical loss early, yet generalises poorly for a long period, followed by an abrupt improvement in test performance after continued optimisation. 
This behaviour is prominent on algorithmic tasks such as modular arithmetic \cite{power2022grokking,progressMeasuresNanda23,liu2023omnigrok,miller2024grokking} and suggests the coexistence of multiple near-zero-loss solution basins with sharply different generalisation, making grokking a concrete instance of the broader basin-selection problem in non-convex optimisation. 
Two complementary questions therefore arise:\ when several basins fit the training data, which basin is statistically preferred, and how does optimisation move between such basins over training time?

Singular learning theory~\cite{watanabe2009slt} (SLT) provides a principled answer to the first question. 
In modern deep learning, two closely related hypotheses are often invoked to explain empirical success: that solutions associated with “flatter” regions of the loss landscape generalise better, and that SGD-based optimisations exhibit inductive biases that favour such regions \cite{hochreiter1997flatminima, keskar2017gengap}. 
While substantial empirical evidence supports these ideas \cite{li2018visneuralnets,jastrzebski2018threefactors,foret2021}, the theoretical foundations remain incomplete. 
For singular models such as neural networks, SLT associates to each minimising basin a local learning coefficient (LLC) $\lambda$, a reparametrisation-invariant measure of local statistical complexity. 
This makes precise, in a coordinate-invariant way, that broader or less complex effective basins are statistically preferred. 
In the local free-energy expansion, $\lambda$ appears as the coefficient of the $\log n$ term, so among near-zero-loss basins with comparable training loss, smaller-LLC basins receive asymptotically larger posterior mass. 
The same quantity also controls the asymptotic Bayes generalisation error.

We use SLT to analyse the geometry of fixed-data loss landscapes, not to identify sample size with training time. 
Classical SLT concerns Bayesian free-energy asymptotics as the sample size grows, whereas grokking occurs at fixed dataset $D$ as SGD moves between competing near-zero-loss basins of the empirical loss $L_D(\theta)$.
The Bayesian lens remains relevant because noisy constant-step SGD can often be approximated by Langevin dynamics or a tempered posterior \citep{mandt2017sgd}, and the LLC quantifies local basin volume and degeneracy. 
Prior work has used LLCs to study Bayesian/SGD phase structure in toy models \citep{chen2023dynamicalversusbayesianphase} and emergent structure in small language models \citep{hoogland2025loss}, but does not derive grokking’s delayed fixed-data training-time transition. 
We therefore separate the statistical question of which basin is preferred in the Bayesian/SLT sense from the dynamical question of when SGD moves between basins, addressing the former theoretically via closed-form analytic formulas and the latter empirically via LLC trajectories.

\paragraph{Contributions.} 

\begin{itemize}
    \item In \cref{sec::LLC in quadractic networks}, we derive closed-form expressions for the local learning coefficients of shallow quadratic networks in both lazy and feature-learning regimes.
    \item In \cref{sec::theory}, we apply these formulas to compare memorising and generalising basins, yielding a precise quantitative distinction between basins that differ in generalisation behaviours. 
    \item In \cref{sec::experiments}, we show that the LLC trajectories correlate with hyperparameter-dependent grokking severity and track the onset of generalisation throughout training, providing empirical evidence for the competing-basin picture underlying the observed transition.
\end{itemize}

\section{Background}\label{sec::background}
\subsection{The Local Learning Coefficient}

\textbf{Local parameter degeneracy and the LLC.}
Deep neural networks are typically \emph{singular}: symmetries, redundant parameterisation, and scaling invariances induce non-identifiability and lead to a rank-deficient Fisher information matrix at local optima. 
Consequently, the local loss landscape is not generically quadratic, the posterior concentration need not be Gaussian, and dimension-based Laplace-penalties can fail.
Singular learning theory~\cite{watanabe2009slt, watanabe2022slt} instead quantifies this local degeneracy via a reparametrisation-invariant measure: the \textit{local learning coefficient}.

Let $(p(x\mid w),q(x),\varphi(w))$ be a model--truth--prior triple, with $w\in W\subseteq\mathbb{R}^d$, and let $L(w)$ be the population loss at $w$.
Given a minimiser $w^*$ with a neighbourhood $U$, define
\begin{equation*} \label{eq: sec2, volume}
    V(\varepsilon):=\int_{U\cap\{w:\,L(w)-L(w^*)\,\le \,\varepsilon\}} \varphi(w)\, dw.
\end{equation*}
As $\varepsilon \to 0$, SLT provides the expansion
\begin{equation*} \label{eq: sec2, free energy}
    V(\varepsilon)
    = c\,\varepsilon^{\lambda(w^\star)}\bigl(-\log\varepsilon\bigr)^{m(w^\star)-1}\\
    +o\left(\varepsilon^{\lambda(w^\star)}(-\log\varepsilon)^{m(w^\star)-1}\right),
\end{equation*}
where $\lambda(w^\star)$ is the LLC and $m(w^\star)$ is a \textit{local multiplicity}. 
In regular models, $L$ is locally quadratic at $w^\star$ and $V(\varepsilon)\asymp \varepsilon^{d/2}$.
Hence $\lambda(w^\star)=d/2$, the classical parameter count notion of complexity of a model.
In singular models, $\lambda(w^\star) < d/2$: a smaller $\lambda$ indicates that many parameter settings near $w^\star$ either minimise the loss exactly or achieve nearly minimal loss. Equivalently, the local effective dimension, as measured by $2\lambda(w^\star)$, is smaller than the ambient parameter dimension $d$ because of parameter degeneracies.
A detailed account of SLT can be found in \cref{appendix: SLT}.

\textbf{Basin selection and the LLC. } 
In Bayesian learning, for a neighbourhood $U$ of $w^*$ and sample negative log-likelihood $L_{n}$,
the LLC controls the logarithmic correction to a basin's \textit{local free energy}:
\begin{equation*}
    F_{n}(U) := - \log \int_{U} \exp\left\{-n L_{n}(w)\right\} \, \varphi(w) \, dw=  n L_{n}(w^*) + \lambda(w^*) \log n \\
    + O(\log \log n).
\end{equation*}
Given two competing solutions $a$ and $b$ with comparable training loss, their local free-energy gap is asymptotically dominated by $(\lambda_a-\lambda_b)\log n$.
Hence, as $n$ increases, the basin with smaller LLC eventually attains lower free energy and therefore greater posterior mass \cite{watanabe2022slt}.
This competition can produce a sharp switch at a critical sample size, corresponding to a first-order Bayesian phase transition from one basin to another~\cite{chen2023dynamicalversusbayesianphase}.

\subsection{Related work}

\textbf{Empirical discovery and mechanistic accounts of grokking.} 
On modular arithmetic, \citet{progressMeasuresNanda23} reverse-engineer a Fourier-feature circuit in a one-layer Transformer and propose progress measures that track circuit formation.
Subsequent work broadens the phenomenon beyond the original algorithmic setting, relating grokking to weight-norm dynamics across images, language, and molecules~\citet{liu2023omnigrok}, and observing analogous behaviour in Gaussian processes, linear regression, and Bayesian neural networks \cite{miller2024grokking}.
A complementary line of work interprets grokking through phase-transition phenomena: 
\citet{rubin2024grokking} map grokking to a first-order transition in two-layer teacher-student models, while \citet{vzunkovivc2024grokking} provide solvable grokking models with analytic critical exponents and time-to-grok distributions.
These results motivate a phase-based view of delayed generalisation, but do not by themselves provide a singular-geometric criterion for comparing competing low-loss basins.

\textbf{Flatness of loss landscape.}
A classical hypothesis in ML is that flatter minima generalise better than sharper ones \cite{Hochreiter1994_flatminima}. 
Many measures of flatness are based on the local curvature of the loss surface, captured by studying the eigenvalues of the Hessian \cite{Kaur2023_maxEigenvalueFlatness}. 
However, these measures are not invariant under general reparametrisations \cite{Dinh_sharp_minima, zhang2021flatnessdoesdoescorrelate}, and modifications to address this problem tend to be limited, e.g. layer-wise invariance \cite{petzka2019reparameterizationinvariantflatnessmeasuredeep}. On the contrary, the LLC is a diffeomorphism-invariant measure of degeneracy \cite{lau2025local}, with strong theoretical and empirical links to Bayesian generalisation \cite{watanabe2009slt}. 
In Appendix \ref{appendix: reparam thm} we show it is invariant under very general re-parametrisations.

\textbf{Singular learning theory and LLC.} SLT analyses singular models (including neural nets) via the Bayes free‑energy expansion with the real log‑canonical threshold (RLCT) as the $\log n$ coefficient; \citet{watanabe2013widely} turns this into practical evidence criteria in singular settings. 
Recent work by \citet{lau2025local} provides scalable posterior estimators of LLC -- a local RLCT capturing the singularity class of a specific basin -- applied to modern neural networks. 
Other recent work applies SLT to interpretability. 
For example, \citet{hoogland2025loss} look at how changes in LLC curves during the training of small transformer-based language models correspond to emerging linguistic capabilities. 
They also employ these curves to understand when certain circuits start forming.

\textbf{Closed-form computations of RLCTs.} Closed-form learning coefficients and their bounds have been derived for a limited number of statistical models where the singular geometry of the loss surface can be resolved explicitly.
Examples include: reduced rank regression~\cite{aoyagi2005stochastic}, where the RLCT depends on the input/output dimension of the model and the rank of the underlying regression matrix; non-negative matrix factorisation (NMF)~\cite{hayashi2017upper}, where bounds depend on the matrix dimensions and the effective factorisation rank; and deep architectures with linear activations~\cite{aoyagi2024consideration,lehalleur2024geometry}, where learning coefficients have been shown to be bounded as network depth increases.
To the best of our knowledge, no closed-form analytic formula of any neural networks with non-linear activations have been previously established.

\section{Problem Set-up}\label{sec::problem set-up}

\subsection{Modular arithmetic task}
Let $p$ be prime and $a,b,c\in\mathbb{Z}_{p}$. 
The modular addition function $f: \mathbb{Z}_{p}\times \mathbb{Z}_{p}\rightarrow \mathbb{Z}_{p}$ is given by $f(a,b)=a+b \mod p$, and is fully described by the collection of triples $\mathcal{D}=\{(a,b,c)\in\mathbb{Z}_{p}^{3}:  c=f(a,b)\}$, where $|\mathcal{D}|=p^2$. 
Viewed as a classification task, a training subset $\mathcal{D}_{\text{N}}:=\{(a_{i},b_{i},c_{i})\}_{i=1}^{N}$ is sampled uniformly without replacement.
Each example-target pair is then encoded as $\mathbf{x}_{i}=[e_{a_{i}};e_{b_{i}}]^{\top}\in \mathbb{R}^{2p}$ and $\mathbf{y}_{i}=e_{c_{i}}\in \mathbb{R}^{p}$, where $e_{i}$ is the standard $i$th basis vector of $\mathbb{R}^{p}$.
The data and target matrices are then $X_{N}=[\mathbf{x}_{1} \dots \mathbf{x}_{N}]\in\mathbb{R}^{2p\times N}$ and $Y_{N}=[\mathbf{y}_{1} \dots \mathbf{y}_{N}] \in \mathbb{R}^{p \times N}$, respectively.

\subsection{Model architecture}

We train a 2-layer quadratic network $f_{\theta}$ with parameters $\theta=(W,V)$, hidden width $K$, and no bias terms such that the predicted labels are given by:
\begin{equation}\label{eq: quadratic model}
    \widehat{Y}_{N}:=f_{\theta}(X_{N})=V\sigma(W^{\top}X_{N}), \quad W\in \mathbb{R}^{d\times K}, \quad V\in\mathbb{R}^{p\times K},
\end{equation}
where $d=2p$ and $\sigma(x)=x^2$ is a quadratic activation. 
We use a regression style $\ell^{2}$ loss function:
\begin{equation*}
    \min_{V,W} \frac{1}{2}\|(Y_{N}-\widehat{Y}_{N})P_{1}^{\perp}\|_{F}^{2},
\end{equation*}
where $P_{\perp}:=I_{N}-\mathbf{11}^{\top}/N$ is the zero-mean projection matrix along the sample dimension and $\|.\|_{F}$ is the Frobenius norm. 
This projection eliminates trivial constant-bias fitting, forcing the model to learn the task's structure and making feature emergence (in the sense of \citet{tian2025provablescalinglawsfeature}) the dominant route to generalisation.

\section{The LLC for Quadratic Networks} \label{sec::LLC in quadractic networks}

In order to compute the LLC in closed form for a QNN, we do not work directly with the parameterisation $\theta=(W,V)$ as it is often not unique.
Since the LLC depends on the local geometry of the loss as a function of $f_{\theta}$, it is convenient to rewrite the network in the following way:
\begin{equation*}
    f_{\theta,k}(x) 
    = x^{\top} \Big( \sum_{j=1}^{K} v_{kj} w_{j} w_{j}^{\top} \Big) x
    = x^{\top} Q_{k} x, \quad x \in \mathbb{R}^{2p}
\end{equation*}
where $Q_{k}\in\textrm{Sym}(\mathbb{R}^{d \times d})$ determines the $k$-th output of the QNN directly.
Equivalently, $f_{\theta}$ induces a parameter map 
\begin{equation*}
    \Phi_{p}(\theta) = \sum_{j=1}^K v_{:j}\otimes (w_j w_j^\top)= (Q_{1},\dots,Q_{p}) \, \in \mathcal{Y},
\end{equation*}
where $\mathcal{Y}:=\mathbb{R}^p\otimes \operatorname{Sym}(\mathbb{R}^{d\times d})$ is the identifiable parameter space.
This naturally defines the affine variety of one-neuron atoms
    \begin{equation*}
        \widehat{\mathcal{X}}
        =
        \{v\otimes(ww^\top) : v\in \mathbb{R}^{p},\; w\in \mathbb{R}^{d}\}
        \subset \mathcal{Y}.       
    \end{equation*}
Hence, a width-$K$ network output is the sum of $K$ elements in $\mathcal{X}$.

At a local optimum $\theta^{*}$, only those perturbations of $\theta$ that change $(Q_{1},\dots, Q_{p})$ can alter the output of the model and hence its loss.
By analysing the image of the Jacobian of $\Phi_{p}$, we identify the locally distinguishable directions of the model and relate them to the secant geometry of the one-neuron model class.
Combined with the reparametrisation theorem for MSE (\Cref{thm: reparam}), this yields the following generic formula for the LLC.
\begin{theorem}\label{thm: generic secant dimension thm (paper)}
    Let $f_\theta$ be a quadratic network with architecture triple $(d,p,K)$ that realises the parameter map $\Phi_{p}(\theta)$ for $\theta=(W,V)$.
    Let $\mathcal{X}=\mathbb{P}(\widehat{\mathcal X}) \subset \mathbb{\mathbb{P}(\mathcal{Y})}$
    denote the projective variety of one-hidden-neuron atoms, and $\widehat{\sigma_K(\mathcal X)}$ denote the affine cone over its $K$-th secant variety.
    Assume the following:
    (1) $w_{j} \not = 0$ and $v_{:j} \not =0$ for all $j=1,\dots, K$;
    (2) the atoms ${x}_{j}:=v_{:j}\otimes(w_{j}w_{j}^{\top}) \in \widehat{\mathcal{X}}$ are in a general position and $y=\sum_{j=1}^{K} x_{j}$ is a generic smooth point on $\widehat{\sigma_{K}(\mathcal X)}$, and;
    (3) the assumptions of the reparametrisation theorem for MSE (\Cref{thm: reparam}) hold at $y=\Phi_{p}(\theta)$.
    Then $\lambda = \frac{1}{2} \dim \widehat{\sigma_{K}(\mathcal X)}.$
\end{theorem}

Determining the LLC therefore reduces to determining the secant dimension of the one-neuron model class.
The ambient identifiable space has dimension $\dim \mathcal{Y} = pD$, where $D := {d(d+1)}/{2}$, while each hidden neuron contributes $d+p-1$ directions in $\operatorname{Im} J \Phi_{p}$ after removal of scaling symmetries.
Therefore, the expected generic secant dimension is $r_{\textrm{exp}}:= \min \, \left(K(p+d-1), \, pD \right)$.
We say that the architecture $(d,p,K)$ is non-defective if $\dim \widehat{\sigma_{K}(\mathcal{X})}= r_{\textrm{exp}}.$ 
Under this assumption, the LLC admits the following explicit form.

\begin{corollary}\label{cor: main LLC corollary (paper)}
    Under the assumptions of \Cref{thm: generic secant dimension thm (paper)}, and assuming that $\widehat{\sigma_{K}(\mathcal{X})}$ is non-defective, the LLC is given by
    \begin{equation*}
        \lambda=
        \begin{cases}
        \dfrac{K(d+p-1)}{2},
        &
        K(d+p-1)<p\,\dfrac{d(d+1)}{2},
        \\[0.8em]
        p\,\dfrac{d(d+1)}{4},
        &
        K(d+p-1)\ge p\,\dfrac{d(d+1)}{2}.
        \end{cases}        
    \end{equation*}
\end{corollary}

\begin{remark}\label{rmk: cor 4.3}
\Cref{cor: main LLC corollary (paper)} is a generic secant-dimension formula.
A special solution, however, may lie on a non-generic stratum of the secant variety where it is represented by only $K_{\textrm{eff.}} < K$ distinct non-zero atoms $v \otimes(ww^{\top})$.
Then the same secant variety argument applies with $K$ replaced by $K_{\textrm{eff.}}$, provided that the reduced representation is itself generic and non-defective and that the solution is locally smooth.
If atoms are non-zero but satisfy non-generic algebraic relations, the local image dimension may be smaller than the generic secant dimension.
Hence, the corollary should be considered as a generic local-complexity reference and not as a classification of all solutions.
Proofs and further geometric discussion are deferred to \Cref{appendix: LLC for QNNs}.
\end{remark}

\section{From Lazy to Feature Learning}
\label{sec::theory}

\begin{table*}[t]
\caption{Summary of derived LLC formulas.}
\label{tab:llc-regimes}
\centering
\small
\setlength{\tabcolsep}{4pt}
\renewcommand{\arraystretch}{1.2}
\begin{tabularx}{\textwidth}{@{}>{\raggedright\arraybackslash}p{0.12\textwidth}
                              >{\raggedright\arraybackslash}p{0.18\textwidth}
                              >{\raggedright\arraybackslash}p{0.16\textwidth}
                              >{\raggedright\arraybackslash}X@{}}
\toprule

\textbf{Regime} & \textbf{Solution type} & \textbf{Result} & \textbf{Role} \\
\midrule
General
& True solution
& \Cref{cor: main LLC corollary (paper)}
& Architecture-level reference: provides generic local complexity near a true solution, independently of the specific optimisation regime. \\
Early
& NTK linearised
& \Cref{thm:main}
& Early-stage baseline when training remains close to initialisation and the network is well approximated by a kernel model. \\
Early
& Lazy-feature
& \Cref{thm: lazy regime}
& Early-stage fixed-representation regime for modular addition, where the hidden features are effectively frozen. \\
Early
& Memorisation
& \Cref{cor:lazy_quad_bounds}
& Early-stage memorising-basin prediction obtained from the lazy-feature analysis. \\
Late
& Feature learning
& \Cref{thm:stage2_llc}
& Late-stage structured-basin prediction for the final trained model; this is the formula tested in Section~6 through its dependence on $p$ and $K$. \\

\bottomrule
\end{tabularx}
\end{table*}

To obtain closed-form predictions on modular addition, we compare the LLC at distinct training equilibria that arise in different regimes of optimisation. 
Following the complementary perspectives of the neural tangent kernel (NTK) approximation~\citep{NTK_jacot} and the staged feature-learning picture of \citet{tian2025provablescalinglawsfeature}, we organize this section into two parts: early-stage fixed-representation approximations, and late-stage feature-learning approximations. 
This organisation mirrors recent theoretical accounts of grokking that emphasize a delayed transition from kernel-/lazy-like dynamics to a rich feature-learning regime~\citep{kumar2023grokking,mohamadi2024you,lyu2023dichotomy,rubin2024grokking}.

\paragraph{Setup.}
Throughout, we use the two-layer architecture $\widehat Y \;=\; F V$ with $F := \sigma(XW)$ as seen in Section \ref{sec::LLC in quadractic networks}, and trained using the loss
\begin{equation*}
J(W,V)
\;:=\;
\frac12\big\|P_1^\perp\big(Y - FV\big)\big\|_F^2
\;+\;\frac{\eta}{2}\big(\|W\|_F^2+\|V\|_F^2\big).
\label{eq:loss_centered}
\end{equation*}
Define the centered quantities $\widetilde Y := P_1^\perp Y$ and $\widetilde F := P_1^\perp F$.
A key object in Tian's analysis is the \emph{backpropagated gradient to the hidden representation},
\begin{equation}
G_F
\;:=\;
-\frac{\partial J}{\partial F}
\;=\;
P_1^\perp (Y - FV)\,V^\top
\;=\;
(\widetilde Y - \widetilde F V)\,V^\top.
\label{eq:GF_exact}
\end{equation}
Intuitively, the transition from memorisation to feature learning is visible in the \emph{structure} of $G_F(t)$ over training: early on it is dominated by noise / idiosyncratic fitting through $V$ (Stage I), and later it develops a task-aligned component that drives coherent updates of $W$ (Stage II).

\paragraph{Standard assumptions.}
Unless otherwise stated, all LLC statements are understood under the same regularity assumptions of~\citet[Theorem~7.1]{watanabe2022slt} as used elsewhere in the paper (i.i.d.\ sampling, appropriate smoothness/analyticity in a neighbourhood of the solution, and a proper prior density that is positive and continuous at the parameter point of interest). 

\subsection{Early-stage approximations: NTK and lazy learning}
Early in training, the neural network can be approximated by two fixed-representation models. First, the Neural Tangent Kernel (NTK) approximation (linearisation in parameters) holds when $\|\theta_t-\theta_0\|$ is small and hence $f_{\theta_t}$ is well-approximated by its first-order Taylor expansion around initialisation~\citep{NTK_jacot}. Second, the Lazy Regime described in~\citet{tian2025provablescalinglawsfeature} is an extended Stage-I phase which occurs where $V$ fits quickly while $W$ changes negligibly. In this case, $F=\sigma(XW)$ is treated as fixed random features and $V$ converges to a centred ridge solution. In this section, we present LLC results under each regime.

\subsubsection*{The NTK regime}

We consider a general model $f_\theta(x)$ and its NTK linearisation around initialisation $\theta_0$. The following result provides an early-stage LLC baseline that applies beyond shallow quadratic networks.

Let
\[
\phi(x)
\;:=\;
\nabla_\theta f_\theta(x)\big|_{\theta=\theta_0}
\in\R^{\tilde d},
\]
where $\tilde d=\dim(\theta)$, and write the linearised model as
\[
f_\theta(x)=f_{\theta_0}(x)+\phi(x)^\top \theta,\qquad \theta\in\R^{\tilde d}.
\]
Assume the true regression function is $f_\star(x)=f_{\theta_\star}(x)$ for some $\theta_\star$.
Let $\ell(y,f)=-\log p(y\mid f)$ be twice continuously differentiable in $f$ and define the population excess risk
\[
\mathcal K(\theta):=\E\!\left[\ell(Y,f_\theta(X))-\ell(Y,f_{\theta_\star}(X))\right]\ge 0,\quad \mathcal K(\theta_\star)=0.
\]
Let
\begin{equation*}
    w(x) :=\E\!\left[\frac{\partial^2 \ell(Y,f)}{\partial f^2}\,\Big|\,X=x,\ f=f_\star(x)\right] \ (>0), \quad
    \mathcal I :=\E\!\left[w(X)\,\phi(X)\phi(X)^\top\right]
\end{equation*}

Write $r:=\mathrm{rank}(\mathcal I)\le \tilde d$.

\begin{theorem}[LLC of NTK models]\label{thm:main}
Assume: (i) $\ell$ is $C^2$ in $f$ and $w(\cdot)$ is bounded away from $0$ and $\infty$ on the support of $P_X$;
(ii) there exists a proper prior density $\varphi$ on $\theta$ with $\varphi(\theta^*)>0$; and (iii) data $(X_i,Y_i)_{i=1}^n$ are i.i.d.\ from $P_\star$ with $f_\star$ as above.
Then the local learning coefficient (real log canonical threshold) of the NTK model at $\theta_\star$ equals
\[
\lambda=\frac{r}{2}.
\]
In particular, if $\mathcal I$ is full rank ($r=\tilde d$), the model is regular and $\lambda=\tilde d/2$; if $r<\tilde d$, the model is singular but still has $\lambda=r/2$.
\end{theorem}

\begin{proof}
A detailed proof can be found in Appendix \ref{Appendix: NTK Theorem}.
\end{proof}
\subsubsection*{Lazy Learning Regime}
We now return to the two-layer model $\widehat Y=FV$ with $F=\sigma(XW)$.
In \citet{tian2025provablescalinglawsfeature}'s Stage-I picture, $F$ behaves like a random representation at initialisation, so the top layer $V$ can fit the training labels quickly while $W$ changes negligibly. 
In this phase, the backpropagated signal $G_F$ in~\eqref{eq:GF_exact} is dominated by noise-like components and does not yet provide a clean learning direction for $W$. 
This is aligned with recent mechanistic explanations of grokking in which the model first fits a near-kernel/random-feature solution before late-time feature learning identifies a generalising solution~\citep{kumar2023grokking,mohamadi2024you}.

If we freeze $W$ and optimize only over $V$ with ridge parameter $\eta$, we obtain the centred ridge solution
\begin{equation}\label{eq:Vridge}
    V_{\mathrm{ridge}}
    \;=\;
    \big(\widetilde F_{\mathrm{train}}^\top \widetilde F_{\mathrm{train}} + \eta I\big)^{-1}\widetilde F_{\mathrm{train}}^\top \widetilde Y_{\mathrm{train}},
\end{equation}
where $\widetilde F_{\mathrm{train}}:=P_1^\perp\sigma(X_{\mathrm{train}}W)$ and $\widetilde Y_{\mathrm{train}}:=P_1^\perp Y_{\mathrm{train}}$ are the centered features/labels on the training set. Write $W=[w_1|\cdots|w_K]$ and $F=[f_1|\cdots|f_K]$ with $f_j=\sigma(Xw_j)$. Under a continuous initialisation for $w_j$, the columns $f_j$ are i.i.d.\ with some law $\mu$ on $\R^n$. Let $L:=\mathrm{span}\big(\mathrm{supp}(\mu)\big)$ and $l:=\dim(L),$ so $l$ is the intrinsic dimension of the feature subspace that the random columns $f_j$ can explore.

\begin{theorem}[LLC in the lazy (random-feature) memorisation regime]
\label{thm: lazy regime}
In the above setting, we work in the interpolation limit $\eta \to 0^+$, in which $V_{\mathrm{ridge}}$ approaches a minimiser of the unregularised loss and the excess empirical-loss contribution to the local free energy becomes negligible. With squared loss and $(X,Y)$ uniformly distributed over a finite dataset, assume the initialisation distribution of each $w_j$ has a density absolutely continuous w.r.t.\ Lebesgue measure. Then the local learning coefficient at the ridge memorisation solution satisfies $\lambda \;=\; \frac{1}{2}p\,\min\{l, K\}$.
\end{theorem}

\begin{proof}
See Appendix~\ref{appendix: lazy regime}.
\end{proof}

\begin{corollary}[Modular arithmetic with quadratic activation]\label{cor:lazy_quad_bounds}
    Let $\sigma(t)=t^2$ and assume we are in the experimentally relevant regime where $2p-1 <K << p^2$. Then, $\tfrac12 p (2p-1) \leq \lambda \leq \tfrac12 pK.$
\end{corollary}
\begin{proof}
See Corollary \ref{cor: quadratic activation} in Appendix \ref{appendix: lazy regime}.
\end{proof}

\subsection{Late-stage approximations: feature learning}

\begin{figure*}[t!]
    \centering
    \begin{subfigure}[t]{0.49\linewidth}
        \centering
        \includegraphics[width=\linewidth]{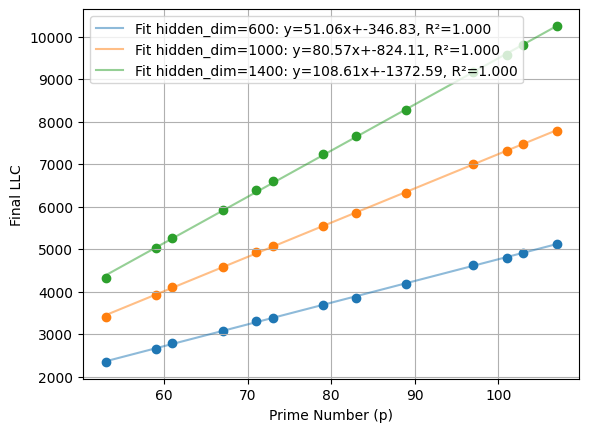}
        \caption{Linear relationship between $p$ and the final LLC of the trained model. Experiment repeated for several values of the hidden-layer dimension.}
        \label{fig:p_vs_llc}
    \end{subfigure}
    \hfill
    \begin{subfigure}[t]{0.49\linewidth}
        \centering
        \includegraphics[width=\linewidth]{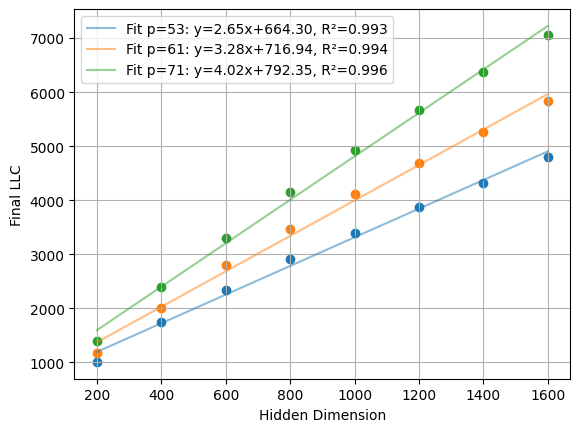}
        \caption{Linear relationship between the hidden-layer dimension and the final LLC of the trained model. Experiment repeated for several values of $p$.}
        \label{fig:hidden_dim_vs_llc}
    \end{subfigure}
    
    \caption{Empirical validation of the linear dependence of the final LLC on architectural parameters.}
    \label{fig:theory_validation_llc}
\end{figure*}

Stage II begins once the backpropagated signal to the hidden representation becomes \emph{task-aligned}. 
In recent grokking theories, this corresponds to ``escaping the kernel regime'' and entering a rich feature-learning phase in which the representation changes sufficiently to uncover a structured generalising solution~\citep{kumar2023grokking,mohamadi2024you,lyu2023dichotomy,rubin2024grokking}.
A consequence is that, for a period of training, neurons are approximately \emph{decoupled} (``independent feature learning''): writing $W=[w_1,\dots,w_K]$ and $F=[f_1,\dots,f_K]$, each neuron tends to climb an energy of the form
\[
\mathcal E(w_j)
\;=\;
\frac12\big\|\widetilde Y^\top f_j\big\|_2^2
\qquad
\Big(f_j=\sigma(Xw_j)\Big),
\]
so local maximizers of $\mathcal E$ correspond to emergent task-aligned features.

Stage II suggests that the relevant late-time solutions are no longer well described by fixed-representation approximations, but instead belong to a structured feature-learning basin of the quadratic network itself. 
We therefore model the late-stage solution directly at the level of the full width-$K$ architecture and apply the generic secant-dimension formula of Corollary~4.2. 
In the subabundant non-defective regime, this gives an explicit prediction for LLC of the late-stage basin.

\begin{corollary}[Stage-II LLC]\label{thm:stage2_llc}
    Assume that the late-stage feature learning solution is a generic true solution $\theta^*$ of the width-$K$ quadratic network where $K(d+p-1) < p \tfrac12 d(d+1)$, and $(d,p,K)$ is a non-defective architecture. Then, by \cref{cor: main LLC corollary (paper)}, the LLC for modular addition ($d=2p$) is $\lambda = \frac{1}{2}K(3p-1).$
\end{corollary}

\begin{remark}
By \cref{rmk: cor 4.3}, the value $\frac{1}{2}K(3p-1)$ is a generic late-stage reference value.
If the trained modular-addition solution is effectively represented by $K_{\textrm{eff.}} < K$ distinct, non-zero atoms then the same argument provides $\lambda = \frac{1}{2}K_{\textrm{eff.}}(3p-1)$ for the subabundant, non-defective regime.
\end{remark}

\section{Experiments}
\label{sec::experiments}
We empirically validate closed-form scaling laws for the local learning coefficient in quadratic networks and demonstrate that LLC trajectories, computed solely from training data, track the emergence of generalisation. In Appendix \ref{appendix: subsection: grokking severity}, we also explain how optimisation hyperparameters modulate grokking severity. The code used for the experiments can be found in the following anonymised repository: \url{https://anonymous.4open.science/r/geom_phase_transitions-DF59/}. The code makes use of the DevInterp package \cite{devinterpcode}.

\subsection{Experimental validation of theoretical scaling laws}

\cref{thm:stage2_llc} predicts linear scaling laws of the LLC $\lambda$ in both $p$ and $K$.
In \cref{fig:p_vs_llc,fig:hidden_dim_vs_llc}, we validate these theoretical predictions for both $p$ and $K$ across a variety of experimental setups.
In~\cref{fig:hidden_dim_vs_llc}, we find that for fixed $p$ the final LLC increases proportionally with network width even though all widths eventually generalise.
This suggests that wider models are not simply ``the small model plus redundant neurons'' (the small solution does not embed as a subnetwork), which is inconsistent with a single width-invariant structured solution where the LLC would be the same across architectures. Detailed descriptions of our experimental setup, including model architecture, training procedures, and hyperparameters can be found in ~\Cref{app: experiments/setup}. 

From the geometric perspective of Section \ref{sec::LLC in quadractic networks}, this behaviour is consistent with the late-stage solution occupying a generic region of the secant variety $\widehat{\sigma_K(X)}$, whose dimension grows linearly with $K$ in the subabundant regime. In this view, increasing width expands the set of independent directions available to the model, so that the dimension of the solution (and hence the LLC) continues to grow with $K$ even after generalisation has been achieved.

\subsection{LLC tracks emergence of generalisation}
\begin{figure*}[t]
    \centering
    \includegraphics[width=\linewidth]{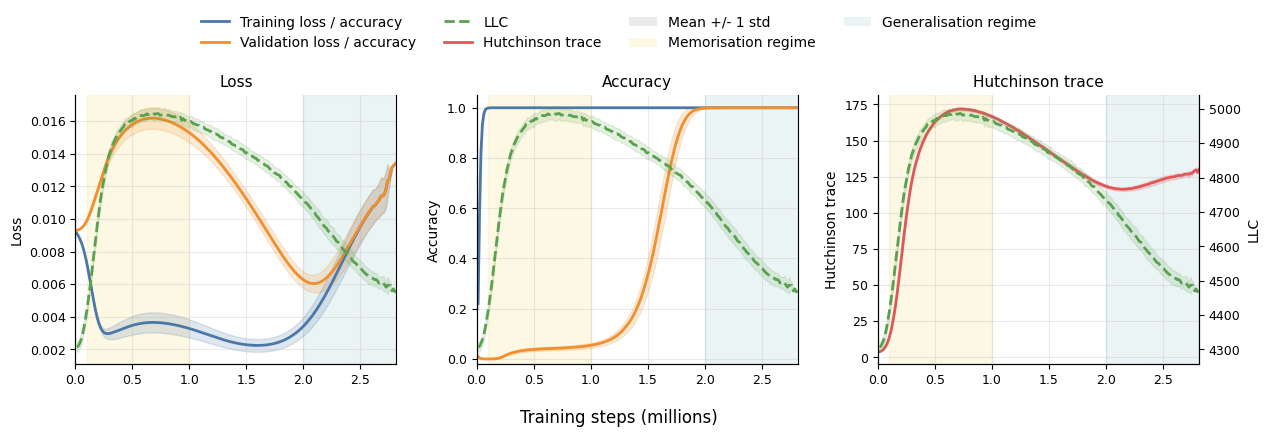}
    \caption{Three plots examining different aspects of the same experiment during training. Results are aggregated over 100 random seeds; solid lines denote the mean across seeds, while shaded regions indicate $\pm 1$ standard deviation. Dataset $p=53$, learning rate 0.0001, weight decay 0.00001, batch size 128 and hidden dimension 1024. The Hutchinson trace estimator is calculated using $m = 100$ random vectors. Left: LLC curve and training/validation loss curves. Centre: LLC curve and training/validation accuracy curves. Right: LLC curve and Hutchinson trace curve.
    }
    \label{fig: main figure}
\end{figure*}
Secondly, we track the training and validation losses, accuracies, and LLC throughout training for a model with standard hyperparameters. Despite the LLC being calculated exclusively from the training data, its evolution closely mirrors that of the validation loss as can be observed in Figure \ref{fig: main figure}. Similar results have been reported in \citet{grokking_slt_blog}. Moreover, in Appendix \ref{app: additional experiments} we show the statistical significance of this experiment by repeating it for different values of $p$, dimension of the hidden layer, and hyperparameters like the learning rate or weight decay.

This behaviour can be understood through the lens of the Bayesian free energy. During the early stages of training, the optimiser primarily reduces the empirical loss, moving rapidly toward regions of low training error. As seen in Figure \ref{fig: main figure}, this initial phase is marked by a rapid collapse in training loss and a near-immediate saturation of training accuracy. However, validation performance remains poor for a long period thereafter, indicating that the optimiser has reached a low-training-loss regime that still generalises poorly. We call this a \textit{memorisation solution.} Once the trajectory reaches a neighbourhood of local near-minimisers, further optimisation is no longer dominated by loss reduction. Instead, the dynamics become increasingly influenced by the local geometry of the loss landscape.

In the delayed-generalisation phase, the LLC rises to a peak and then declines, while training accuracy remains essentially unchanged and validation accuracy increases sharply. 
This is consistent with the optimiser moving within the low-loss landscape from a memorising, geometrically more complex region toward a geometrically simpler one which we call a \textit{generalising solution.} 
In this regime, stochastic optimisation is therefore consistent with a preference for regions of lower LLC, corresponding to more degenerate, higher-volume minima.
To compare this singular-geometric picture with a classical curvature-based diagnostic, we also plot the Hutchinson trace, which estimates $\operatorname{tr}(H)$ and hence the average second-order variation of the loss (see \cref{appen:hutchinson}).
In \cref{fig: main figure}, the Hutchinson trace and the LLC are closely aligned until the end of the grokking phase, with both peaking prior to the generalisation transition and subsequently decreasing.
However, once training and validation accuracies have saturated, the two probes diverge:
the LLC continues to decrease until reaching a plateau, whereas the Hutchinson trace begins to increase.
We interpret this as a post-generalisation regime in which continued optimisation, influenced by weight decay, further simplifies the internal representation while sharpening the remaining identifiable directions in parameter space, without changing the learned classifier. 
Thus, the model complexity drops in the SLT sense, yielding a lower LLC, while becoming sharper along active directions, yielding a higher Hessian trace.
This late-stage divergence shows that the Hutchinson is a useful curvature proxy during the grokking transition, but that curvature alone is not a complete measure of singular complexity.
Finally, although the LLC is estimated using only training data, it effectively captures geometric properties of the loss landscape that govern out-of-sample performance.

\section{Conclusion}\label{sec::conclusion}
In this work, we studied grokking as a phase transition between competing near-zero-loss basins with distinct LLCs. We showed theoretically for quadratic models how distinct solutions can exhibit different LLCs, and empirically used LLC to explain some features of grokking. More broadly, this work suggests that SLT-based quantities such as the local learning coefficient can serve as informative probes of training dynamics in over-parametrized models, linking loss-landscape geometry, implicit regularisation, and generalisation behaviour. Extending this perspective to other architectures (e.g. ReLU networks and transformers) and to a wider class of grokking tasks remains an important direction for future work.

\textbf{Limitations.} 
Our analysis is conducted in a Bayesian asymptotic setting and provides a characterisation of basin selection, rather than a direct analysis of SGD dynamics. 
While this perspective is supported by prior work ~\cite{mandt2017sgd,chen2023dynamicalversusbayesianphase} and our empirical results, a complete theoretical connection between posterior concentration and stochastic gradient-based optimisation remains open. 
In addition, our results are established for simplified model classes to permit explicit analysis; extending them to more complex architectures and training regimes would be a direction for future work.
Finally, the central result in \Cref{thm:stage2_llc} is to be interpreted as a generic reference value: trained endpoints may violate assumptions of genericity, or the architecture may be defective, leading to a lower LLC than predicted.

\bibliographystyle{plainnat}
\bibliography{references}

\appendix
\newpage
\section{Singular Learning Theory}\label{appendix: SLT}

\subsection{Introduction to Singular Learning Theory}
Singular Learning Theory (SLT) is at its heart the theory of singularities in the parameter space of parametric models. It blends algebraic geometry with an underlying Bayesian framework to try to understand some statistical phenomena of singular models. The aim of this section is to give a basic overview of the subject to make some of the concepts in the paper more accessible to readers who are first encountering the wonderful world of SLT. 
\begin{definition}[Model Triplet]
    Let $W$ be a compact space of parameters with a prior distribution $\varphi (w)$. Consider a parametric model with density $p(x|w)$ and a true data-generating mechanism $q(x)$. Then, 
    $$ (p(x|w), q(x), \varphi(w)),$$
    is called a \textbf{model-truth-prior triplet}.
\end{definition}
\begin{definition}[Fisher Information Matrix]
    For a given statistical model $p(x|w)$, the \textbf{Fisher information matrix} is given by $I(w)=[I_{jk}(w)]$, where
    $$ I_{jk}(w) := \int \bigg(\frac{\partial \log p(x|w)}{\partial w_j}\bigg)\bigg(\frac{\partial \log p(x|w)}{\partial w_k}\bigg)p(x|w)dx.$$
\end{definition}
\begin{definition}[Regular/Singular Model]
    A statistical model $p(x|w)$ is said to be \textbf{regular} if 
    \begin{itemize}
        \item[(i)] the Fisher information matrix $I(w)$ is positive definite, and
        \item[(ii)] the model is identifiable, that is, if the function $w \mapsto p(\cdot |w)$ is injective.
    \end{itemize}
    If a model is not regular, then it is strictly singular. Finally, the set of \textbf{singular models} comprises both regular and strictly singular models.
\end{definition}
\textbf{Remark.} Most of the results in classic statistical learning theory assume that the working model is regular. Results like the Cramer-Rao inequality, the asymptotic normality of the Bayes posterior distribution around the unique parameter $w_0$ such that $q(x)=p(x|w_0)$, or the quadratic expansion of the Kullback-Leibler divergence, are all properties of regular models. However, most modern architectures, including layered neural networks or mixture models, are not regular. Hence, all of these nice properties do not hold and the question is now what results can be adapted to singular models. Singular learning theory argues that some of the questions can be answered by understanding certain singularities in the space of parameters $W$.
\begin{definition}[Kullback-Leibler Divergence]
    The \textbf{Kullback-Leibler divergence} between two probability measures $q,p$ is given by 
    $$ \mathrm{KL}(q||p):= \mathbb{E}_{X\sim q}\bigg[\log \frac{q(X)}{p(X)}\bigg] = \int q(x) \log \frac{q(x)}{p(x)} dx.$$
\end{definition}
Intuitively, the Kullback-Leibler divergence is a measure of how different the two probability measures are. Notice that it is in general not symmetric. However, $\mathrm{KL}(q||p)\geq 0$ with $\mathrm{KL}(q||p)= 0$ if and only if $q(x)=p(x)$ almost surely with respect to the Lebesgue measure. In our setting we are interested in the divergence between the statistical model and the true data-generating density. That is, $K(w) := \mathrm{KL}(q(x) || p(x|w) )$. Then, we are interested in studying the set of \textbf{optimal parameters}
$$ W_0 := \{ w \in W : K(w) = K_0\} \ \ \text{with} \ \ K_0 = \inf_{w'\in W} K(w')$$
For regular models, $W_0 = \{w_0\}$, but in singular models the model is not identifiable so we have to treat $W_0$ as a more general analytic variety. In order to study this analytic variety, we look at the zeta function as defined below.
\begin{definition}[Zeta function and RLCT]
    The \textbf{zeta function} is defined for Re$(z) > 0$ by 
    $$ \zeta (z) := \int_W (K(w)-K_0)^z \varphi(w) dw.$$
    This function can be analytically continued to a meromorphic function on the complex plane. All of the poles are real, negative and rational. Let $-\lambda$ be the largest pole of $\zeta$ and $m$ its multiplicity. Then, the \textbf{Real Log Canonical Threshold (RLCT)} and  \textbf{multiplicity} of the model-truth-prior triple are precisely $\lambda$ and $m$. The RLCT is the name given to the quantity $\lambda$ in the algebraic geometry literature. It plays an important role there as a birational invariant. In the machine learning literature, it is called the \textbf{Learning Coefficient}.
\end{definition}
Watanabe proposes a method in his seminal work \cite{watanabe2009slt} to compute the RLCT via \textit{resolution of singularities}. Importantly, it is precisely this RLCT which determines the geometry of the analytic variety $W_0$, which in turn controls some interesting statistical phenomena of the statistical model. 

We can understand the learning coefficient from a geometric lens. Theorem 7.1 in  \citet{watanabe2009slt} relates the learning coefficient $\lambda$ with a notion of flatness of the parameter space. Indeed, define the volume function 
$$ V(\varepsilon) := \int_{\{w:K(w)-K_0<\varepsilon\}} \varphi(w)dw.$$

Then, the mentioned theorem gives the following asymptotic expansion of the prior mass of parameters whose loss lies within $\varepsilon$ of a minimum:

$$ V(\varepsilon) = c\varepsilon^\lambda  (-\log \varepsilon)^{m-1} + o(\varepsilon^\lambda  (-\log \varepsilon)^{m-1}),$$
as $\varepsilon \to 0$.

\begin{figure}[t]
    \centering
    \includegraphics[width=\linewidth]{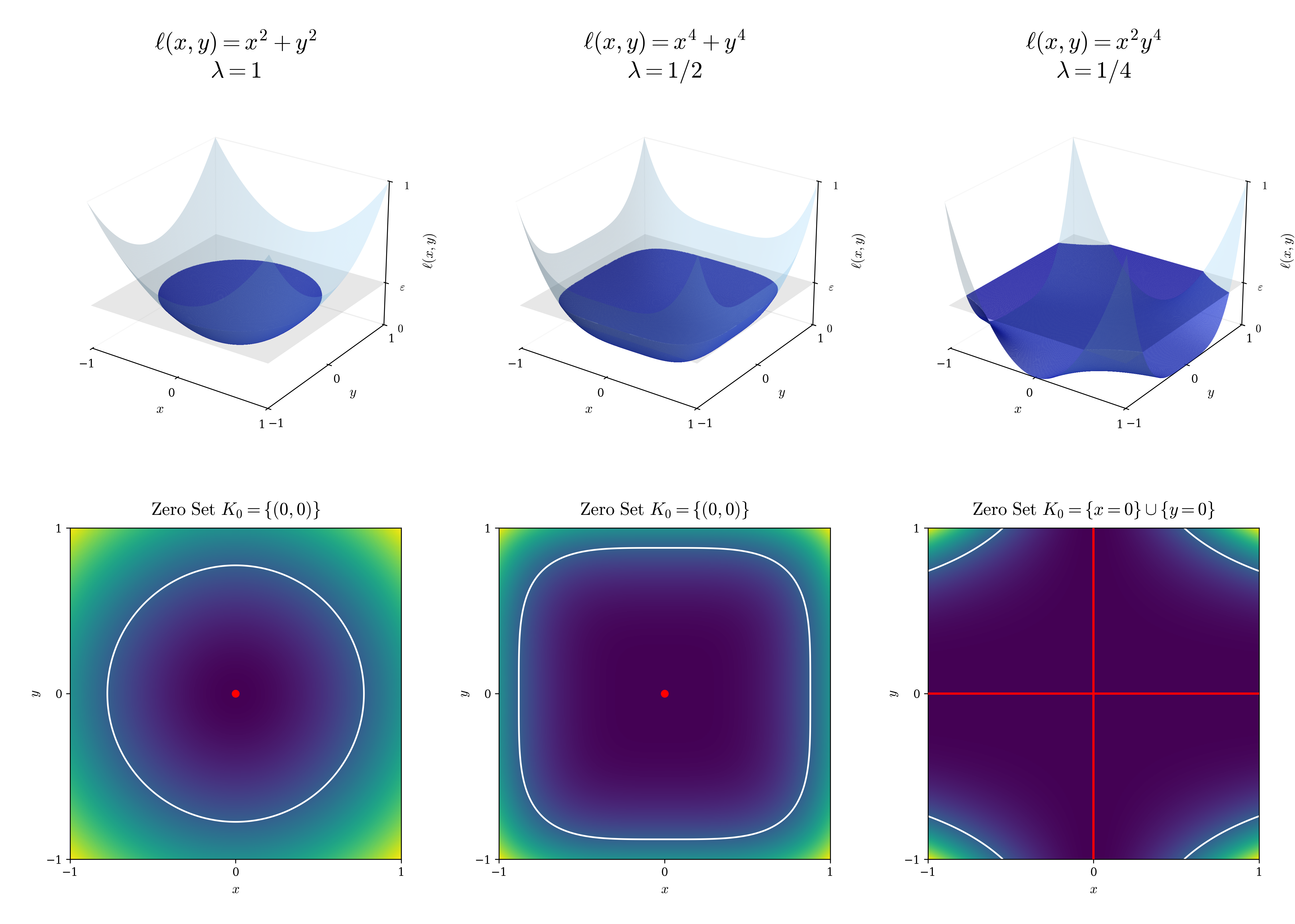}
    \caption{Three different loss landscapes are realised alongside their contour maps in the $(x,y)$-plane.
    Although the first two loss functions share the same zero-set, their volume-sets scale differently with $\varepsilon$ and, hence, their local learning coefficients are not equivalent.
    The final loss function shows that the zero-set need not be a single point.}
    \label{fig:placeholder}
\end{figure}

This gives a geometric intuition as to what the learning coefficient quantifies. It is a measure of the degeneracy of the loss landscape (which is equivalent in geometry to the KL landscape). Smaller values of $\lambda$ correspond to more degenerate, or flatter, regions. Again, this tells us that intuitively, $\lambda$ is a measure of the effective number of parameters the model has. This intuition can be made more rigorous by exploring other asymptotic results concerning the learning coefficient. To make this concrete, let us introduce some final terminology:
\begin{definition}
    For a given statistical model $p(x|w)$, define the following terms:
    \begin{itemize}
        \item Let $D:=\{x_i\}_{i=1}^n$ be i.i.d. samples from $q(x)$. This is our \textbf{dataset}.
        \item Let \textbf{sample negative log likelihood} be defined as 
        $$ L_n(w) = -\frac{1}{n}\sum_{i=1}^n \log p(x_i|w).$$
        This can be thought of as the \emph{training error} if we let $p(x_i|w) \propto e^{-l(x_i,w)}$, where $l(x_i,w)$ is the loss for data point $x_i$. We also define the \textbf{population loss} as $L(w) = \mathbb{E}_{X\sim q}[-\log p(X|w)]$. Notice then that the training error can be seen as an instance of an unbiased estimator of the population loss. 
        \item The \textbf{marginal likelihood} is defined as
        \begin{align*}
            Z_n & := \int_W \prod_{i=1}^n p(x_i|w) \varphi(w) dw \\ &= \int_W \exp \{-n L_n(w)\} \varphi (w)dw.
        \end{align*} 
        \item The \textbf{free energy}, or negative log marginal likelihood, is then just defined as $F_n = - \log Z_n$.
        \item The\textbf{ predictive distribution} $p(x|D_n)$ is defined as 
        $$ p(x|D_n)=\int_W p(x|w)p(w|D_n)dw.$$
    \end{itemize}
\end{definition}
It is worth making explicit here that mathematical results tend to be derived by analysing geometrically the variety $K(w)$, whereas many papers in the literature with a more empirical flavour discuss the loss landscape, or $L_n(w)$. Notice that 
\begin{align*}
    K(w) & = \mathbb{E}_{x\sim q} \left [ \log\frac{q(x)}{p(x|w)} \right ] \\ 
    & =  \mathbb{E}_{x \sim q} \left [ \log q(x)\right ] -\mathbb{E}_{x \sim q} \left [ \log p(x|w)\right ] \\
    &=  H(q) + L(w).
\end{align*}
So, the population loss only differs from the KL divergence by a constant, known as the \textbf{entropy}.

\subsubsection*{Statistical properties of models}
We now briefly explain some results about the statistical properties of the model. First of all, let us unpack what the free energy actually encodes. From Bayes' rule we know that:
\begin{align*}
    p(w|D_n)=\frac{p(D_n|w)\varphi (w)}{Z_n}
\end{align*}
Take logs and expectation over $w \sim p(w|D_n)$ now.
\begin{align*}
    F_n & = -\log Z_n \\
    & = \mathbb{E}_{w \sim p(w|D_n)} \left [ nL_n(w) +  \log p(w|D_n) - \log \varphi (w)\right] \\
    &  =  \mathbb{E}_{w \sim p(w|D_n)} \left [ nL_n(w)\right] + K(p(w|D_n) || \varphi (w)).
\end{align*}
Therefore, the free energy gives a measure of the expected training loss under the posterior distribution plus the divergence between this same posterior and the prior.

In \citet{watanabe2009slt}, Watanabe gives an asymptotic expansion of the free energy: given a $w_0\in W_0$, we can asymptotically expand the free energy as 
$F_n = nL_n(w_0) + \lambda \log n -(m-1)\log\log n + o_p(1)$.

This elucidates why phase transitions can be tracked by changes in $\lambda$. 
\\
Furthermore, the \textbf{Bayes generalisation error} is also related to the learning coefficient. This is given by 
$$ B_g = \mathrm{KL}(q(x) || p(x|D_n)),$$
so it encapsulates the difference in the distributions between the true data generating distribution $q(x)$ and the prediction distribution $p(x|D_n)$. Then, Watanabe showed as well in \citet{watanabe2009slt} that its expectation also has an asymptotic expansion given by the following:
$$ \mathbb{E}[B_g] =  \frac{\lambda}{n} + o\bigg(\frac{1}{n}\bigg).$$

This result is very significant. Together with the geometric interpretation of the learning coefficient, it provides evidence to the widely accepted idea in the machine learning community that flatter regions of the loss landscape lead to models with better generalisation \cite{petzka2021relative}. 

\subsection{Local Learning Coefficient}
In this section we briefly explain the paper \citet{lau2025local}. Up to now, we have discussed asymptotic results on the global minimisers of the loss. In \citet{lau2025local}, they extend the definition of the learning coefficient to local minimisers, adding a component of practicality to the theory, since we rarely have access to global minimisers. 

Let $w^*\in W$ be a local minimum of the population loss $L(w)$. Define the ball 
$$ B(w^*,\varepsilon) = \{w \in B(w^*) : L(w)-L(w^*) < \varepsilon\},$$
where $B(w^*)$ is a closed ball around $w^*$ for which $w^*$ is a minimiser. Define now 
$$ V(\varepsilon) = \int_{B(w^*,\varepsilon)} \varphi (w) dw.$$
This is essentially the same definition as in the learning coefficient but we only integrate around a closed ball of $w^*$. Then, we still get the same asymptotic expansion as before 
$$ V(\varepsilon) = c\varepsilon^{\lambda(w^*)}  (-\log \varepsilon)^{m-1} + o(\varepsilon^{\lambda(w^*)}  (-\log \varepsilon)^{m-1}),$$
as $\varepsilon \to 0$, for some rational value $\lambda(w^*)$ which is called the \textbf{local learning coefficient (LLC)}. In \citet{lau2025local}, they prove that the LLC is invariant under local diffeomorphisms. It turns out that this definition is equivalent to that of the learning coefficient if we restrict the parameter space to $B(w^*)$ with a normalised prior induced by $\varphi (w)$. 

The advantage of the LLC is that it allows for a computation of a metric which gives local information of the geometry of the landscape, and still encodes information about the statistical properties of the model. This quantity has consistent estimators which can be efficiently computed, which is the content of the next section. 
\section{LLC estimation} \label{appendix: llc estimation}

In this section, we provide the algorithmic details of estimating local learning coefficients in neural network architectures. We follow the sampling algorithm as implemented in \citet{devinterpcode}. Firstly, we present the theory behind the LLC estimator and then we conduct some experiments to understand how to calibrate some of its hyperparameters.

\subsection{LLC estimation: theory}
\paragraph{Setup.}
Let $\{(x_i,y_i)\}_{i=1}^n$ be the dataset and $L_n(\theta)=\frac{1}{n}\sum_{i=1}^n \ell(\theta;x_i,y_i)$ the empirical loss. Fix a local minimizer $\hat\theta\in\arg\min_\theta L_n(\theta)$.
For inverse temperature $\beta>0$, define the tempered Gibbs density
\[
p_\beta(\theta)\ \propto\ \exp\!\big(-n\beta\,L_n(\theta)\big)\,\varphi(\theta),
\]
with a smooth prior $\varphi$ positive near $\hat\theta$.

Let $Z_n(\beta)=\int \exp(-n\beta L_n(\theta))\varphi(\theta)\,d\theta$ and $F_n(\beta)=-\beta^{-1}\log Z_n(\beta)$.

Denote by $\lambda$ the RLCT and by $m\in\mathbb N$ its multiplicity at $\hat\theta$.
As $n\beta\to\infty$,
\begin{equation}
Z_n(\beta)\ \asymp\ e^{-n\beta L_n(\hat\theta)}\,(n\beta)^{-\lambda}\,\{\log(n\beta)\}^{\,m-1}\cdot C\,(1+o(1)).
\label{eq:Z-asymp}
\end{equation}
Taking logarithms,
\begin{equation}
\log Z_n(\beta)
\approx -\,n\beta L_n(\hat\theta) - \lambda \log(n\beta) + (m-1)\log\log(n\beta) + \mathrm{const} + o(1).
\label{eq:logZ}
\end{equation}

Using $\frac{\partial}{\partial\beta}\log Z_n(\beta) = -\,\mathbb E_\beta[n L_n(\theta)]$, we have
\begin{align}
\mathbb E_\beta[L_n(\theta)]
&= -\frac{1}{n}\frac{\partial}{\partial\beta}\log Z_n(\beta)
\approx L_n(\hat\theta) + \frac{\lambda}{n\beta} - \frac{m-1}{n\beta\log(n\beta)} + o\!\Big(\frac{1}{n\beta}\Big).
\end{align}
Hence
\begin{equation}
n\beta\big(\mathbb E_\beta[L_n(\theta)] - L_n(\hat\theta)\big)
\approx \lambda - \frac{m-1}{\log(n\beta)} + o(1).
\label{eq:llc-observable}
\end{equation}
We define the \emph{Local Learning Coefficient} (LLC) as the limit of the left-hand side; for regular minima in $d$ dimensions one has $\lambda=d/2$.

\paragraph{Estimator.}
Draw $\theta_1,\dots,\theta_T$ from $p_\beta$ (e.g.\ via SGMCMC) and compute $L_t=L_n(\theta_t)$.
Let $L_{\mathrm{init}}:=L_n(\hat\theta)$.
The (chain-wise) estimator is
\[
\widehat{\lambda}_\beta \;:=\; n\beta\Big(\frac{1}{T}\sum_{t=1}^T L_t - L_{\mathrm{init}}\Big),
\]
and we average across chains.
A bias-reduced estimate is obtained by regressing $\widehat{\lambda}_{\beta_j}$ against $1/\log(n\beta_j)$ across multiple temperatures and taking the intercept.

We proceed to discuss the meaning and selection of hyperparameters, categorized by hyperparameters in the target-distribution and sampling process. 

Target–distribution hyperparameters: 
\begin{itemize}
    \item \textit{Inverse temperature $n\beta$} controls locality of the target. Larger $n\beta$ concentrates sampling near $\hat\theta$, resulting in a more local estimation and lower bias (since bias $\propto 1/\log n\beta$), but reduces movement which leads to higher MC variance.
    Too small $n\beta$ risks crossing basins, then the LLC estimation is no longer local.
    \item \textit{Weight decay $\varphi(\theta)$.} $-\log\varphi(\theta)$ is added to the potential $U(\theta)=n\beta\,L_n(\theta)-\log\varphi(\theta)$.
    In the implementation, it is preferred to use a weak, smooth prior matching training's weight decay to avoid distorting the local geometry.
    \item \textit{Temperature sweep $\{\beta_j\}$.} We evaluate $\widehat{\lambda}_{\beta_j}$ at several $n\beta_j$ and regress $\widehat{\lambda}_{\beta_j}$ against $1/\log(n\beta_j)$; use the intercept as a bias–reduced estimate of $\lambda$.
    \item \textit{Optional localisation coefficient $\gamma$. }Adds $\frac{\gamma}{2}\|\theta-\hat\theta\|^2$ to $U(\theta)$ to enforce locality.
    It stabilizes runs but changes the target; prefer increasing $n\beta$ before using $\gamma>0$.
\end{itemize}

Sampler hyperparameters: 
\begin{itemize}
    \item \textit{Langevin/SGMCMC step size $\eta$ (\texttt{lr}).}
    Discretization step of the dynamics. Larger $\eta$ improves exploration but increases integrator bias and instability; smaller $\eta$ is stable but may freeze (LLC $\approx 0$).
    
    \item \textit{Number of chains $C$ (\texttt{num\_chains}).} Independent replicas for variance estimation and convergence checks; typical $C\in\{4,8\}$.
    
    \item \textit{Draws per chain $T$ (\texttt{num\_draws}).} Samples used to estimate $\overline{L}$; typical $T\in[200,1000]$. Effective sample size (ESS) matters more than raw $T$.

    \item \textit{Burn–in $B$.} Discarded initial steps before collecting draws; typical $B\in[100,500]$, check stationarity via traces.
    
    \item \textit{Thinning $k$.} Keep every $k$–th draw to reduce autocorrelation. Prefer reporting ESS rather than heavy thinning.
\end{itemize}

\subsection{LLC estimation: effect of hyperparameters}

    \begin{figure}[H]
    \centering
      \begin{subfigure}[t]{0.32\linewidth}
        \centering
        \includegraphics[width=\linewidth]{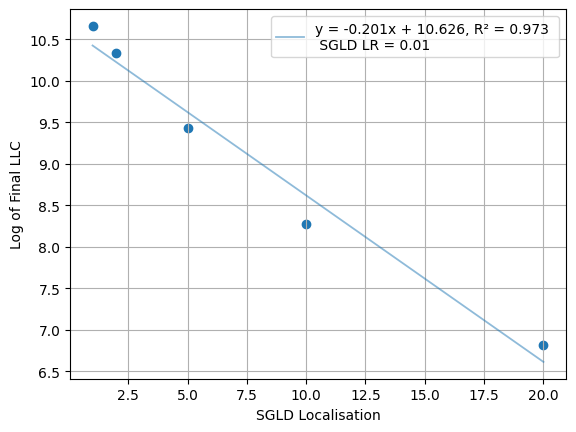}
        \caption{ SGLD learning rate = 0.01}
      \end{subfigure}\hfill
      \begin{subfigure}[t]{0.32\linewidth}
        \centering
        \includegraphics[width=\linewidth]{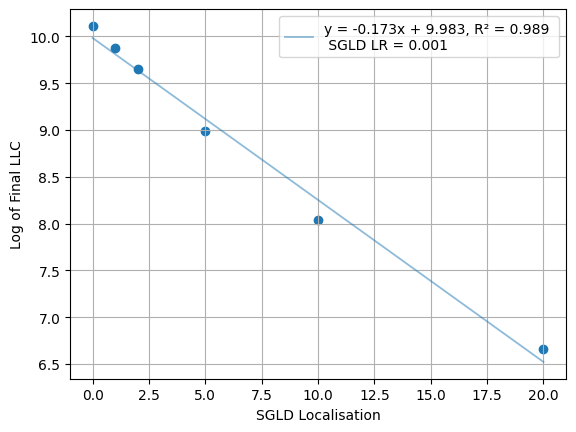}
        \caption{SGLD learning rate = 0.001}
      \end{subfigure}\hfill
      \begin{subfigure}[t]{0.32\linewidth}
        \centering
        \includegraphics[width=\linewidth]{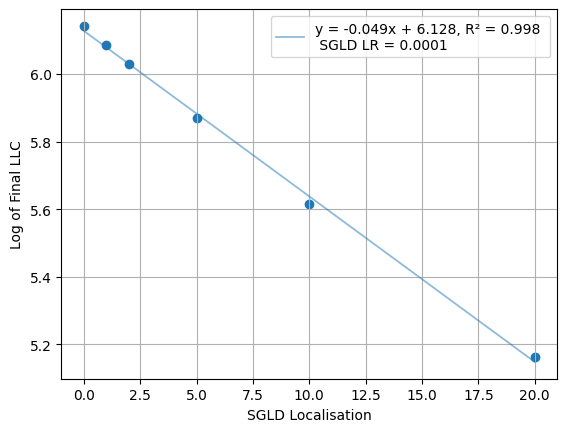}
        \caption{ SGLD learning rate = 0.0001}
      \end{subfigure}
      \caption{Log-linear relationship between the final LLC of a trained model and the SGLD localisation parameter $\gamma$, shown for different fixed values of the SGLD learning rate.}
    \end{figure}

To understand the effect of the localisation $\gamma$ and SGLD learning rate hyperparameters on the LLC estimators, we trained exactly the same model and compared the final LLC estimate while only varying the SGLD sampler. The estimates varied by orders of magnitude. In fact, we observed a log-linear relationship between $\gamma$ and the estimated LLC across different SGLD learning rates. The log-linear relationship is consistent with theory: the parameter volume grows exponentially as the SGLD sampler becomes less localised at the minimum. The essential conclusion from these experiments is that, (i) it is complicated to draw exact comparisons between theoretical LLC predictions and estimated LLC values, and (ii) to draw comparisons between different LLC estimates, it is vital to maintain these hyperparameters constant.

\section{SGLD for posterior sampling}

Stochastic optimization is traditionally analyzed in terms of convergence to a minimizer. Under the small‑step, small‑batch regime, SGD performs a noisy discretization of Langevin diffusion whose equilibrium law is a tempered Bayesian posterior \citep{welling2011bayesian, mandt2017sgd}. 

Let $\mathcal D=\{(x_i,y_i)\}_{i=1}^{n}$ be the training set and
\begin{equation}\label{eq:nll-posterior}
\mathcal L(\theta)\;=\;
          -\tfrac1n\sum_{i=1}^{n}\log p(y_i\mid x_i,\theta)
          \;+\; \lambda \|\theta\|_2^{2}
\end{equation}
with negative log‑posterior under an isotropic Gaussian prior $\theta\sim\mathcal N(0,(2\lambda)^{-1}I)$. A single mini‑batch update of size $B$ with learning rate $\eta$ is
\begin{equation}\label{eq:sgd-discrete}
\theta_{k+1}
      = \theta_k
        - \eta\,\widehat{\nabla}\mathcal L(\theta_k),
\qquad
\widehat{\nabla}\mathcal L(\theta_k)
        = \nabla\mathcal L(\theta_k)+\xi_k,
\end{equation}
where $\xi_k$ is the stochastic gradient noise.

Assume that (\emph{i}) mini‑batches are sampled uniformly without replacement, (\emph{ii}) $B\ll n$, and (\emph{iii}) the true gradient covariance $C(\theta)=\frac1n\sum_{i=1}^{n}g_i g_i^{\!\top}$ with
$g_i=\nabla_\theta\bigl[-\log p(y_i\mid x_i,\theta)\bigr]$ varies slowly in~$\theta$. Then $\xi_k$ satisfies
$\mathbb E[\xi_k]=0$ and
$\operatorname{Cov}(\xi_k)\approx\frac{n-B}{nB}\,C(\theta_k).$
For $\eta\ll1$ the Euler–Maruyama limit of \eqref{eq:sgd-discrete} is the \emph{pre‑conditioned Langevin SDE}
\begin{equation}\label{eq:langevin}
d\theta_t
  = -\nabla\mathcal L(\theta_t)\,dt
    + \sqrt{\,2T\,\Sigma(\theta_t)}\,dW_t,
\qquad
T=\frac{\eta(n-B)}{2nB},
\quad
\Sigma(\theta)=C(\theta),
\end{equation}
where $W_t$ is a Wiener process. If $\Sigma(\theta)\equiv I$ (or is replaced by its equilibrium value), \eqref{eq:langevin} reduces to \emph{Langevin dynamics} with invariant density
\begin{equation}\label{eq:gibbs}
\pi_T(\theta)
   \;\propto\;
   \exp\!\bigl[-n\,\mathcal L(\theta)/T\bigr]
   \;=\;
   p(\mathcal D\mid\theta)^{1/T}\,p(\theta).
\end{equation}
Thus stochastic gradient descent is formally a \textbf{Gibbs sampler} at \emph{temperature} $T\propto\eta/(nB)$. Deterministic gradient descent appears as the \emph{zero‑temperature} limit $T\to0$ (full batch or vanishing step size).

Under standard dissipativity and non‑degenerate noise assumptions, the SDE~\eqref{eq:langevin} is ergodic, and
the empirical law of the discrete iterates converges to $\pi_T$ in Wasserstein‑2 distance at a rate $W_2\bigl(\text{law}(\theta_k),\pi_T\bigr)
      = O\!\bigl(\sqrt{\eta}+B^{-1}\bigr)$.

\section{Proof of NTK theorem} \label{Appendix: NTK Theorem}
We first conduct a second-order Taylor expansion of the population risk, which is
\[
\mathcal K(\theta):=\E\!\left[\ell(Y,f_\theta(X))-\ell(Y,f_{\theta_\star}(X))\right]\ge 0,\quad \mathcal K(\theta_\star)=0,
\]
where the model in question is the Neural Tangent Kernel approximation, that is
\[
f_\theta(x)=f_{\theta_0}(x)+\phi(x)^\top \theta,\qquad \theta\in\R^{\tilde d}.
\]
First notice that 
$\frac{\partial f_\theta(x)}{\partial \theta} = \phi(x).$
Hence, using the chain rule, we find that 
$$\frac{\partial l(Y,f_\theta(X))}{\partial \theta} = \frac{\partial l(Y,f_\theta(X))}{\partial f_\theta(X)}\frac{\partial f_\theta(x)}{\partial \theta} = \frac{\partial l(Y,f_\theta(X))}{\partial f_\theta(X)} \phi(X). $$
Because $f_{\theta*}$ is a population risk minimiser, we find that 
$$ \mathbb{E} \left [\frac{\partial l(Y,f_\theta(X))}{\partial \theta} \bigg |_{\theta=\theta*} \right ]= \mathbb{E} \left [ \frac{\partial l(Y,f_\theta(X))}{\partial f_\theta(X)} \bigg |_{\theta=\theta*} \phi(X) \right ]=0.$$
Secondly, let us look at the Hessian:
\begin{align*}
     \frac{\partial^2 l(Y,f_\theta(X))}{\partial \theta \partial \theta^T} &= \frac{\partial}{\partial \theta} \bigg ( \frac{\partial l(Y,f_\theta(X))}{\partial f_\theta(X)^T} \phi(X)^T\bigg) \\
     &= \frac{\partial^2 l(Y,f_\theta(X))}{\partial f_\theta(X)\partial f_\theta(X)^T} \phi(X) \phi(X)^T 
\end{align*}
Finally, the second-order Taylor expansion becomes:
\begin{align*}
    \mathcal K(\theta) &=\E\!\left[\ell(Y,f_\theta(X))-\ell(Y,f_{\theta_\star}(X))\right] \\
    & = \E\left[ \frac{\partial l(Y,f_\theta(X))}{\partial \theta} \bigg |_{\theta=\theta*}^T (\theta-\theta*) + \frac{1}{2}(\theta-\theta*)^T \frac{\partial^2 l(Y,f_\theta(X))}{\partial \theta \partial \theta^T}(\theta-\theta*) + o(||\theta-\theta*||^2) \right] \\
    & = \frac{1}{2}(\theta-\theta*)^T \E \left[ \frac{\partial^2 l(Y,f_\theta(X))}{\partial f_\theta(X)\partial f_\theta(X)^T} \phi(X) \phi(X)^T\right](\theta-\theta*) + o(||\theta-\theta*||^2) \\
    & = \frac{1}{2}(\theta-\theta*)^T \mathcal{I}(\theta-\theta*) + o(||\theta-\theta*||^2)
\end{align*}
Let $\mathcal I=Q\Lambda Q^\top$ with orthogonal $Q$ and $\Lambda=\mathrm{diag}(\lambda_1,\dots,\lambda_r,0,\dots,0)$, $\lambda_j>0$.
In coordinates $\delta:=\theta-\theta_\star=Q(z,w)$ with $z\in\R^r$ (range) and $w\in\R^{\tilde d-r}$ (null),
\[
\begin{aligned}
    \mathcal K(\theta) &=\tfrac12\, z^\top \Lambda_r z + o(\|z\|^2+\|w\|^2),
\quad \Lambda_r=\mathrm{diag}(\lambda_1,\dots,\lambda_r). \\
 &=\tfrac12\, z^\top \Lambda_r z + o(\|z\|^2),
\end{aligned}
\]
since $\mathcal K(\theta)$ is independent of $w$.
The centered marginal likelihood (dropping constants that contribute $O_p(1)$) satisfies
\[
Z_n \ \asymp\ \int_{\R^{\tilde d}} \exp\!\big(-n\,\mathcal K(\theta)\big)\,\varphi(\theta)\,d\theta.
\]
Dominated by a neighbourhood of $u_\star$, we use the local quadratic approximation and integrate out the null-space coordinates. Writing $\delta=\theta-\theta_\star=Q(z,w)$, we obtain
\[
Z_n \asymp \int_{\mathbb{R}^r} \exp\!\left(-\frac n2 z^\top \Lambda_r z\right)\, m(z)\,dz,
\]
where
\[
m(z):=\int_{\mathbb{R}^{\tilde d-r}} \varphi(\theta_\star+Q(z,w))\,dw.
\]
Since $\varphi$ is a proper prior and $m(z)$ is finite and continuous at $z=0$ under the stated regularity assumptions, we have $m(z)=m(0)+o(1)$ near $0$. Therefore,
\[
Z_n = m(0)\int_{\mathbb{R}^r}\exp\!\left(-\frac n2 z^\top \Lambda_r z\right)\,dz\,(1+o(1)).
\]
The Gaussian integral is
\[
\int_{\mathbb{R}^r}\exp\!\left(-\frac n2 z^\top \Lambda_r z\right)\,dz
=
(2\pi)^{r/2}n^{-r/2}(\det \Lambda_r)^{-1/2}.
\]
Hence
\[
Z_n = C\,n^{-r/2}(1+o(1))
\]
for some finite constant $C>0$, and so
\[
\log Z_n = -\frac r2 \log n + O(1),
\]
which proves that $\lambda = r/2$.
\section{Proof of Lazy Regime Theorem}
\label{appendix: lazy regime}

The aim of this section is to prove theorem \ref{thm: lazy regime}. Here is the proof:
\begin{proof}
We use $X$ for both the random variable with uniform distribution over the whole dataset and the whole data matrix $X$. Then,
\begin{align*}
    \ell(Y, \hat{Y}_{\mathrm{ridge}}) & = \tfrac12 \| Y - \hat{Y}_{\mathrm{ridge}} \|^2_F \\
    &= \tfrac12 \| \operatorname{vec}(Y) - (I \otimes F)\operatorname{vec}(V_{\mathrm{ridge}}) \|^2. \\
\end{align*}
For simplicity let all the lowercase versions of the matrices be the column-wise vectorisation of the corresponding matrices, eg. $y=\operatorname{vec}(Y)$. Now, 
\[\mathbb{E} \left [ \frac{\partial \ell}{\partial v}\Big|_{V = V_{\mathrm{ridge}}} \right ]
= \mathbb{E} \left [ (I \otimes F)^\top \big( y - (I \otimes F) v \big)\Big|_{v = v_{\mathrm{ridge}}} \right].\]
Since $v_{\mathrm{ridge}}$ minimises the regularised objective $J(v) = \ell(v) + \tfrac{\eta}{2}\|v\|^2$, we have $\nabla J(v_{\mathrm{ridge}})=0$, hence
\[
\mathbb{E}\!\left[\frac{\partial \ell}{\partial v}\Big|_{v=v_{\mathrm{ridge}}}\right]
= -\eta\, v_{\mathrm{ridge}}.
\] 
We work in the interpolation limit $\eta \to 0^+$, in which $v_{\mathrm{ridge}}$ approaches a minimiser of the unregularised loss and the first-order term becomes negligible. Next,
\[
\frac{\partial^2 \ell}{\partial v \partial v^\top}\Big|_{v = v_{\mathrm{ridge}}}
= (I \otimes F)^\top (I \otimes F).
\]
Hence, we may Taylor expand the excess population risk as:
\begin{align*}
    \mathcal{K}(v)
=& \ \mathbb{E}\!\left[\tfrac12 (v - v_{\mathrm{ridge}})^\top
    (I \otimes F)^\top(I \otimes F)\,
    (v - v_{\mathrm{ridge}})\right] \\
    & \quad + o(\|v-v_{\mathrm{ridge}}\|^2).
\end{align*}

Therefore, by using the same argument presented in the NTK theorem, we can say that
\begin{align*}
    \lambda & = \tfrac12 \operatorname{rank} \mathbb{E}\!\left[ (I \otimes F)^\top(I \otimes F) \right] \\
    & = \tfrac12 p \operatorname{rank} \mathbb{E} [F^\top F].
\end{align*}

Now, the expectation is taken over $X$, which has a uniform distribution. The dataset is finite, so the expectation is precisely the mean of all possible values of $F^\top F$. This means that if we let, by abusing notation, 
$X$ be the whole data matrix and $F$ the corresponding resulting matrix, then we have
\begin{align*}
    \lambda & = \tfrac12 p \operatorname{rank} \{ \tfrac{1}{p^2} F^\top F \} \\ 
    & = \tfrac12 p \operatorname{rank} \{ F^\top F \} \\ 
    & = \tfrac12 p \operatorname{rank} \{F\}.
\end{align*}

Finally, since $\sigma$ is real analytic and $W$ is initialised with a distribution whose density is absolutely continuous with respect to the Lebesgue measure, we can use proposition \ref{theorem:rank_after_activation} from this appendix.
\end{proof}

\paragraph{Notation.}
For $m\in\mathbb N$, let $\lambda_m$ denote Lebesgue measure on $\mathbb R^m$.
A random vector $U\in\mathbb R^m$ is \emph{absolutely continuous} if it admits a density
$f_U$ with respect to $\lambda_m$.

Throughout, $W=[w_1,\dots,w_K]\in\mathbb R^{d\times K}$ denotes a random matrix with
\emph{independent} columns $w_j$ such that each $w_j$ is absolutely continuous on $\mathbb R^d$. In most cases, the matrix of parameters $W$ is randomly initialised with respect to an absolutely continuous density such as the normal distribution, so this assumption is realistic. 

Let $X\in\mathbb R^{n\times d}$ be deterministic with $\operatorname{rank}(X)=r$,
and let $W\in\mathbb R^{d\times K}$ have independent, absolutely-continuous columns. Let $\sigma$ be real-analytic. Define the random feature matrix
\[
F := \sigma(XW)\in\mathbb R^{n\times K},
\]
where $\sigma$ is applied entry-wise. Also, if $\mu$ is the law of a column of $F$, define
$$ \mathcal{L} = \operatorname{span} \{\operatorname{supp}(\mu)\},$$
and let $l = \dim \mathcal{L}$.

\subsection{Main result}
The aim of this appendix is to prove the following proposition:
\begin{theorem}[Almost-sure rank after an entry-wise activation]
\label{theorem:rank_after_activation}

Using the above notation and assumptions, the following holds:
\[
\operatorname{rank}(F)=\min(l,K)\qquad\text{almost surely.}
\]
\end{theorem}
We are going to build the proof sequentially. That is, lemma \ref{lem:generic_full_rank} shows that $W$ has full rank almost surely. Then, lemma \ref{lem:rank_after_linear_map} shows that under the linear map $X$, the columns of $W$ lie in a space of dimension $\operatorname{rank}(X)$, inside a possibly bigger ambient space. Finally, we study the effect of applying $\sigma$ entry-wise. The non-linearity of the operator can curve the space and increase its dimension significantly. We can only hope to produce bounds under assumptions of what $\sigma$ looks like. 

\subsection{Generic full rank under absolute continuity}

\begin{lemma}[Generic full rank for random matrices]
\label{lem:generic_full_rank}
Let $W=[w_1,\dots,w_K]\in\mathbb R^{d\times K}$ have independent columns, each absolutely
continuous on $\mathbb R^d$. Then
\[
\operatorname{rank}(W)=\min(d,K)\qquad\text{almost surely.}
\]
\end{lemma}

\begin{proof}
First assume $K\le d$. The event $\{\operatorname{rank}(W)<K\}$ is the event that the columns
are linearly dependent. By symmetry it suffices to show
\[
\mathbb P\!\bigl(w_K\in \operatorname{span}\{w_1,\dots,w_{K-1}\}\bigr)=0.
\]
Condition on $(w_1,\dots,w_{K-1})=(v_1,\dots,v_{K-1})$. Then
$\operatorname{span}\{v_1,\dots,v_{K-1}\}$ is a linear subspace of $\mathbb R^d$ of dimension
at most $K-1<d$, hence it has $\lambda_d$-measure zero. Since $w_K$ is absolutely continuous,
\[
\mathbb P\!\bigl(w_K\in \operatorname{span}\{v_1,\dots,v_{K-1}\}\mid w_1=v_1,\dots,w_{K-1}=v_{K-1}\bigr)=0.
\]
Taking expectation over $(w_1,\dots,w_{K-1})$ yields the claim.

Now assume $K>d$. Then $\operatorname{rank}(W)\le d$ always. Let $\widetilde W$ be the
$d\times d$ submatrix formed by the first $d$ columns of $W$. By the case $K=d$ above,
$\operatorname{rank}(\widetilde W)=d$ almost surely, hence $\operatorname{rank}(W)\ge d$
almost surely. Therefore $\operatorname{rank}(W)=d=\min(d,K)$ almost surely.
\end{proof}

\subsection{After a fixed linear map}

\begin{lemma}[Rank after a fixed linear map]
\label{lem:rank_after_linear_map}
Let $X\in\mathbb R^{n\times d}$ be deterministic with $\operatorname{rank}(X)=r$.
Let $W\in\mathbb R^{d\times K}$ satisfy the assumptions of Lemma~\ref{lem:generic_full_rank}.
Then
\[
\operatorname{rank}(XW)=\min(r,K)\qquad\text{almost surely.}
\]
\end{lemma}

\begin{proof}
Let $Z:=XW=[Xw_1,\dots,Xw_K]\in\mathbb R^{n\times K}$.
Choose $U\in\mathbb R^{n\times r}$ with orthonormal columns spanning $\operatorname{Im}(X)$,
so that $U^\top:\operatorname{Im}(X)\to\mathbb R^r$ is an isometry and
\[
\operatorname{rank}(Z)=\operatorname{rank}(U^\top Z)
\qquad\text{($U^\top$ is injective on $\operatorname{Im}(X)$).}
\]
Define $\widetilde X:=U^\top X\in\mathbb R^{r\times d}$, which has full row-rank $r$.
Then $U^\top Z=\widetilde X W$.

Fix $j\in\{1,\dots,K\}$. Since $\widetilde X$ has full row-rank, there exists a right-inverse
$B\in\mathbb R^{d\times r}$ with $\widetilde X B=I_r$.
Hence, for any $w_j \in \mathbb{R}^d$, it can be decomposed into $$w_j = Bu_j + v_j,$$ where $u_j:=\widetilde X w_j\in\mathbb R^r$ and $v_j\in\ker(\widetilde X)$.
The map $w_j\mapsto u_j=\widetilde X w_j$ is a surjective linear map, and because $w_j$ has a
density on $\mathbb R^d$, its pushforward $u_j$ is absolutely continuous on $\mathbb R^r$
(one can verify $\mathbb P(u_j\in A)=0$ for every $A\subset\mathbb R^r$ with $\lambda_r(A)=0$
by pulling back $A$ to $B A\oplus \ker(\widetilde X)$, which has $\lambda_d$-measure zero).

Moreover, the $u_j$ are independent because they are measurable functions of the independent
$w_j$. Hence the columns of $\widetilde XW=[u_1,\dots,u_K]$ are independent and absolutely
continuous in $\mathbb R^r$. By Lemma~\ref{lem:generic_full_rank} (in dimension $r$),
$\operatorname{rank}(\widetilde XW)=\min(r,K)$ almost surely, and therefore
$\operatorname{rank}(XW)=\min(r,K)$ almost surely.
\end{proof}

\subsection{Proof of main result}
\begin{proof}[Proof of theorem \ref{theorem:rank_after_activation}]
Since $\operatorname{rank}(X)=r$, Lemma~\ref{lem:rank_after_linear_map} tells us that the columns of $Z = [z_1, \ldots, z_K]$ lie inside an r-dimensional subspace $\operatorname{Im}(X) \subseteq\R^n$. Recall that we called the target space $L := \operatorname{span}\{\sigma(z) | z \in \operatorname{Im}(X)\}$ of dimension $l$ and $F = [\sigma(z_1), \ldots , \sigma(z_K)]$. From the definition of these spaces it is clear that 
$$ \operatorname{rank}(F) \leq \min(l,K).$$
We now show by an inductive argument that this is in fact an equality. Let
$$ L_j := \operatorname{span}\{\sigma(z_1), \ldots, \sigma(z_j)\},$$
and assume that this has rank $j<\min(l,K)$. Therefore $L_j$ is a proper subspace of $L$, so there exists a vector $a$ orthogonal to $L_j$ but not to the whole of $L$. Define the map
\[
\begin{aligned}
f:\operatorname{Im}(X) &\longrightarrow \mathbb{R} \\
z &\longmapsto a^\top \sigma(z).
\end{aligned}
\]
Since $\sigma$ is a real-analytic function, then the whole of $f$ is real-analytic too. Moreover, by definition of $a$, it is not identically equal to $0$. By the Identity Theorem, if $f$ is not identically 0 on an open set, then its zero-set $\{z\in \operatorname{Im}(X) | f(z)=0\}$ must have Lebesgue measure 0 too. Since $z_{j+1}$ is absolutely continuous by Lemma \ref{lem:rank_after_linear_map}, then with probability 1 it will lie outside the zero-set. That is,
$$ \mathbb{P}(\sigma(z_{j+1}) \in L_j ) = 0.$$
Hence, $\operatorname{rank}(L_{j+1}) = j+1$ almost surely, concluding the inductive argument. 
\end{proof}

\begin{remark}[Increase of rank due to $\sigma$]
Activation functions are generally non-linear. This means that it maps the linear subspace spanned by the columns of $XW$ into a curved space, whose linear dimension is generally much greater.
For example, take $n=2$, $d=1$, $X=(1,2)^\top$ (so $r=1$), $K=2$, and $\sigma(t)=e^t$.
Then $XW$ has rank $1$ (all columns are multiples of $(1,2)^\top$), but
$\sigma(XW)$ has columns $(e^{w_j},e^{2w_j})^\top=(u_j,u_j^2)^\top$ which are generically
not collinear, so $\operatorname{rank}(\sigma(XW))=2$ with positive probability.
\end{remark}

\subsection{Polynomial activation function}
\begin{proposition}\label{prop: sigma(x)=ax^s}
    Let $\sigma(t) = at^s$ for $a\neq0$. Then, 
    $$ r \leq l \leq \min \{n , \binom{r+s-1}{s}\}.$$
\end{proposition}
\begin{proof}
    Let $Z\subset\mathbb{R}^n$ be an $r$-dimensional linear subspace and let $A\in\R^{n\times r}$ be a basis matrix such that 
    $$ Z = \{z: z=Au, u \in \mathbb{R}^r\}.$$
    Set
    \[
    M:=\binom{r+s-1}{s}.
    \]
    Then there exist a (polynomial) feature map $\Phi:\R^r\to\R^M$ and a matrix $C\in\R^{n\times M}$
    such that for all $u\in\R^r$,
    \[
    \sigma(Au)=C\,\Phi(u).
    \]
    Indeed, order coordinates in $\R^M$ by multi-indices $\alpha\in\mathbb{N}^r$ with $|\alpha|=s$ and define the
    degree-$s$ monomial feature map
    \[
    \Phi(u):=\bigl(u^\alpha\bigr)_{|\alpha|=s}\in\R^M.
    \]
    Write $a_i^\top\in\R^{1\times r}$ for the $i$-th row of $A$. By the multinomial theorem,
    \[
    (a_i^\top u)^s=\sum_{|\alpha|=s}\binom{s}{\alpha}\,a_i^\alpha\,u^\alpha,
    \]
    where, for $\alpha = (\alpha_1, \ldots, \alpha_r)$, then we have $a_i^\alpha \, u^\alpha = (a_{i1}u_1)^{\alpha_1} \cdots (a_{ir}u_r)^{\alpha_r}$. Hence, define $C\in\R^{n\times M}$ by $C_{i,\alpha}:=a\binom{s}{\alpha}a_i^\alpha$. Then the $i$-th
    coordinate of $C\Phi(u)$ equals $a(a_i^\top u)^s=\sigma((Au)_i)$, so $\sigma(Au)=C\Phi(u)$. Therefore $\sigma(Z)=\{\sigma(Au):u\in\R^r\}\subseteq\mathrm{Im}(C)$ and
    $\dim\mathrm{span}(\sigma(Z))\le \dim\mathrm{Im}(C)=\mathrm{rank}(C)\le \min(n,M)$.
\end{proof}

\begin{corollary}\label{cor: quadratic activation}
    Let $\sigma(t)=t^2$ and assume we are in the experimentally relevant regime where $K << p^2$. Then, 
    $$\tfrac12 p (2p-1) \leq \lambda \leq \tfrac12 pK.$$
\end{corollary}
\begin{proof}
Since $\sigma(t)=t^2$ is piecewise injective, \Cref{thm: lazy regime} applies. By Proposition~\ref{prop: sigma(x)=ax^s} and $\mathrm{rank}(X)=2p-1$ for modular addition inputs, we have
\[
2p-1=r \;\leq\; l \;\leq\; \binom{r+2-1}{2} \;=\; \frac{r(r+1)}{2} \;=\; p(2p-1).
\]
However, to represent the modular addition structure one needs at least $K\gtrsim 2p-1$ feature degrees of freedom (matching the rank constraint). In the experimentally relevant regime $K\ll p^2$, the bounds in \Cref{cor:lazy_quad_bounds} simplify to
\[
\frac12\,p(2p-1)
\;\leq\;
\lambda
\;\leq\;
\frac12\,pK.
\]
\end{proof}
\section{Reparametrisation Theorems}\label{appendix: reparam thm}
\begin{theorem}\label{thm: reparam}
    Let $W$ be a parameter space with $W_0 = \{w_0\in W : K(w_0) = \inf_{w\in W} K(w)\}$, where $K(w)=D_{KL}(q(x) \, \| \,p(x|w)).$ Let $w_0 \in W_0$ such that the prior $\varphi(w)$ is non-zero. Also, let $\phi: W \to V$ be a reparametrisation such that the Jacobian of $\phi$ at $w_0$ is of full rank. Let the Kullback--Leibler divergence $\tilde{K}$ in the new space $V$ satisfy $K(w)=\tilde{K}(\phi(w))$. Then, the local learning coefficient of the model at $w_0$ is equal to that of the re-parametrised model at $\phi(w_0)$:
    $$ \lambda_{w_0} = \lambda_{\phi(w_0)}.$$
\end{theorem}
\begin{corollary}\label{cor: reparam LLC}
    Let $\phi:W\to V$ be a re-parametrisation as in the above theorem, such that the resulting re-parametrised model is regular. Then $\lambda_{w_0}=\frac{dim(V)}{2}$.
\end{corollary}

\begin{proof}[Sketch proof]
    Since $\phi$ is a submersion at $w_0$ (i.e. has full-rank Jacobian), the submersion theorem implies that in a neighbourhood of $w_0$ the map is locally equivalent to a projection $(v,n)\mapsto v$. This allows us to apply the co-area theorem, which rewrites the zeta integral as a fibre integral of the form 
$$ \int_W K(w)^z\,dw \;=\; \int_V K(\phi^{-1}(v))^z \int_{\phi^{-1}(v)} \tfrac{1}{J_\phi(w)}\,dV_{\phi^{-1}(v)}(w)\, dV_V(v).$$ 
Because the Kullback--Leibler divergence satisfies $K(w)=\tilde{K}(\phi(w))$, it is constant along each fibre $\phi^{-1}(v)$, so the inner integral contributes only an analytic factor $\rho(v)$. Analytic weights cannot create or destroy poles of the zeta function, and therefore the poles are determined entirely by the factor $\tilde{K}(v)^z$ near $v=\phi(w_0)$. Hence the local learning coefficient (the largest pole of the zeta function) is invariant under such a reparametrisation.
\end{proof}

Let us formalise the above. The essential theorem is given by the co-area theorem for manifolds \cite{co-areaTheorem}. To this end, we consider two general Riemannian manifolds $X,Y$ of dimensions $n+k$ and $n$ respectively, equipped with Riemannian metrics $g_X$ and $g_Y$. Denote by $|dV_X|$ and $|dV_Y|$ be the canonical volume densities induced by the Riemannian metrics. 
\begin{theorem}\label{thm: co-area}
    Suppose that $\phi: X \to Y$ is a $C^1-$map such that at any $x\in X$ the differential $D_x\phi: T_xX \to T_{\phi(x)}Y$ is surjective. Denote by $J_\phi(x)$ the Jacobian of the map. Then, for any non-negative function $f:X\to \mathbb{R}$ which is measurable with respect to $|dV_X|$ we have
    \begin{equation}
        \int_{x\in X} f(x)|dV_X| = \int_{y\in Y} \bigg ( \int_{x \in \phi^{-1}(y)} \frac{f(x)}{J_{\phi}(x)}|dV_{\phi^{-1}(y)}(x)| \bigg ) |dV_Y(y)|,
    \end{equation}
    where $|dV_{\phi^{-1}(y)}(x)|$ denotes the volume density on the fibre $\phi^{-1}(y)$ induced by restriction of the Riemannian metric on the fibre. 
\end{theorem}
\noindent \textbf{Remark} This can be viewed as an extended version of Fubini. Indeed, Fubini claims that under certain conditions we have 
$$ \int_{\mathbb{R}^{n+k}}f(x^1,\ldots, x^{n+k})dx^1\cdots dx^{n+k} = \int_{\mathbb{R}^{n}} \bigg( \int_{\mathbb{R}^{k}}f(x^1,\ldots, x^{n+k})dx^{n+1}\cdots dx^{n+k} \bigg )dx^1\cdots dx^{n}. $$
Now, we can recast this into a formula where we integrate along fibres like in theorem \ref{thm: co-area}. Indeed, let $\phi: (x^1,\ldots, x^{n+k}) \mapsto (x^1,\ldots, x^n)$ be the projection map. We can write this as $\phi(x,y)=y$. Then, Fubini can be written as
$$ \int_{\mathbb{R}^{n+k}}f(x,y)dxdy = \int_{\mathbb{R}^{n}} \bigg( \int_{\phi^{-1}(y)}f(x,y)dx \bigg )dy. $$
How is this intuition extended to manifolds? If we have a map $\phi:X\to Y$ whose Jacobian is surjective, the submersion theorem tells us that, locally, this looks like a projection $\phi: (x,y)=(x_1,\ldots, x_{n+k}) \mapsto (x_1,\ldots, x_n)=y$. We just need to ensure that when we choose charts with respect to which the map $\phi$ looks locally like a submersion, the contribution of this change of coordinates is included in the integral.
\\

That is, if our original integral is over $X$ but we now wish to integrate over $Y$, then we need to change the density at each point $y\in Y$ so that it incorporates the contribution of the integral along the fibres $\phi^{-1}(y)$. But since $\phi$ looks locally like a projection, this means integrating along $x_1,\ldots,x_n$ and $x_{n+1},\ldots,x_{n+k}$ separately. 

\begin{proof}[Proof of theorem \ref{thm: reparam}]
    Without loss of generality, let us assume that $K(w_0)=0$. Then, recall that the local learning coefficient of the model at the point $w_0$ is the largest pole of the zeta function $$ \zeta (z) = \int_{w\in U} K(w)^z\varphi(w)dw, $$
    where $U\subset W$ is an open neighbourhood of $w_0$. Then, we are precisely in the setting of the co-area theorem (theorem \ref{thm: co-area}) so we can recast the integral as follows.
    \begin{align}
        \zeta (z) & = \int_{w\in U} K(w)^z\varphi(w)dw \\
        & = \int_{v\in V} \bigg ( \int_{w\in \phi^{-1}(v)} K(w)^z \varphi(w)\frac{|dV_{\phi^{-1}(v)}(w)|}{J_{\phi} (w)} \bigg ) |dV_V(v)| \\
        & = \int_{v\in V} \tilde{K}(v)^z \bigg ( \int_{w\in \phi^{-1}(v)} \varphi(w)\frac{|dV_{\phi^{-1}(v)}(w)|}{J_{\phi} (w)} \bigg ) |dV_V(v)| \\
        & = \int_{v\in V} \tilde{K}(v)^z \rho(v) |dV_V(v)|,
    \end{align}
    where
    \begin{itemize}
        \item $K = \tilde{K} \circ \phi$, i.e. $K(w) = \tilde{K}(\phi(w))$, where $\tilde{K}$ is the Kullback--Leibler divergence expressed in the re-parametrised space $V$. 
        \item $ \rho(v) = \int_{w\in \phi^{-1}(v)}\varphi(w)\frac{|dV_{\phi^{-1}(v)}(w)|}{J_{\phi} (w)} $ measures the contribution to the zeta integral of the fibre $\phi^{-1}(v)$. 
    \end{itemize}
    So it is enough to show that $\rho (\phi(w_0)) \neq 0$. Then, the largest pole of this zeta function in the re-parametrised space corresponds precisely to the local learning coefficient of the model at the point $\phi(w_0)$. But, since $\phi$ is a submersion, we know that $J_{\phi}(w_0)>0$ and its continuity ensures that this is positive in a neighbourhood of $w_0$. Hence, since the LLC only depends on the local structure of the manifold, we can pick an arbitrarily small neighbourhood in which the integral is non-zero. 

    Finally, since in the re-parametrised space the model is regular we obtain $\lambda_{w_0} = \lambda_{\phi(w_0)} = \frac{dim(V)}{2}$
    
\end{proof}

\noindent \textbf{Remark I.} The above theorem is useful to find the LLC of models which are \textit{singular}, but for which the manifold $W_0$ does not contain singularities. If this is the case, you don't need to resolve the singularities using Hironaka's resolution of singularities theorem. You only need to re-parametrise the model so that the resulting model is regular. Then you can just read off the LLC by looking at the dimension of the re-parametrised space. Notice that if you have singularities, no re-parametrisation can get rid of them. 
\\ 

\noindent \textbf{Remark II.} We can think of a \textit{strict reparametrisation} of a model $p(x|w)$ as a map $\phi:W\to V$ such that there exists a function $g$ with $g(x|\phi(w))=p(x|w)$ for all $x,w$. However, the theorem requires a weaker assumption, namely just that the re-parametrised model gives rise to the loss or Kullback-Leibler divergence. That is, that there exists a $g$ such that 
\begin{align*}
    K(w) &:= \mathrm{KL}(q(x)||p(x|w)) \\
    & = \mathrm{KL}(q(x)\|g(x|\phi(w))) =: \tilde{K}(\phi(w)).
\end{align*} 

\section{The Local Learning Coefficient for Quadratic Neural Networks.}\label{appendix: LLC for QNNs}

\subsection{Single Output Regression Models}

To compute the LLC of a quadratic network, it is convenient to work in identifiable coordinates rather than in the raw parameterisation $\theta=(W,V,b,c) \in \Theta$.
Indeed, writing
\begin{equation}\label{eq: quadratic regressor}
    f_{\theta}(x)=\sum_{j=1}^{K} v_{j} (w_{j}^{\top} x + b_{j})^{2} +c = x^{\top} Q x + 2r^{\top}x + s
\end{equation}
we obtain the macro parameters
\begin{equation*}
    Q := \sum_{j=1}^{K} v_{j} w_{j} w_{j}^{\top} \in \operatorname{Sym} (\mathbb{R}^{d \times d}),
    \qquad
    r := \sum_{j=1}^{K} v_{j} b_{j} w_{j} \in \mathbb{R}^{d},
    \qquad
    s := \sum_{j=1}^{K} v_{j} b_{j}^2 + c \in \mathbb{R}.
\end{equation*}
So, $f_{\theta}$ defines an algebraic sub-family of quadratic regression models on $x$ with rank constraint $\operatorname{Rank (Q)} \leq K$.
Furthermore, \Cref{eq: quadratic regressor} defines the identifiable parameter map
\begin{equation*}
    \Phi(W,V,b,c):=\left(Q,r,s \right) \in \mathcal{M}_{K},
\end{equation*}
where
\begin{equation*}
    \mathcal{M}_{K}:= \left \{(Q,r,s) \in \mathcal{V} \, \vert \, \operatorname{Rank}(Q) \leq K, r\in \textrm{col}(Q) \right\}, \quad \mathcal{V}:=\operatorname{Sym} (\mathbb{R}^{d \times d}) \times \mathbb{R}^{d} \times \mathbb{R}.
\end{equation*}
Only perturbations in $\theta$ that result in a change in $(Q,r,s)$ can affect the function $f_{\theta}$ realised by the network and, hence, the local geometry of the loss landscape.
The LLC is therefore determined by the image of the Jacobian $J \Phi(\theta^*)$, where $\theta^*$ is a local optimum.
This provides the general proof strategy for the following two theorems: we study the image of the Jacobian $\Phi$ in order to determine precisely how many directions in the identifiable parameter space contribute to an optimal solution's LLC, i.e. determine the rank of $J \Phi(\theta^*)$ or, equivalently, $\dim \operatorname{Im} J \Phi(\theta^*)$.

\paragraph{Over-Parametrised Regime} 
We first consider the over-parametrised regime where $K \geq d$.

\begin{theorem}[LLC for quadratic network with $K \geq d$]\label{thm: LLC QNN regressor over-param}
    Let $d \in \mathbb{N}$ and $K\geq d$, and suppose $f_{\theta}$ is trained with MSE loss, achieving a true solution at $\theta^{*}=(W^{*},V^*,b^*,c^*)\in \Theta$ such that $\textrm{Rank} (Q^{*})=d$. 
    Then the local learning coefficient of the corresponding network is given by
    \begin{equation*}
        \lambda = \frac{(d+1)(d+2)}{4}.
    \end{equation*}
\end{theorem}

\begin{proof}
    Let 
    \begin{equation*}
        \mathcal{V}:= \operatorname{Sym} (\mathbb{R}^{d \times d}) \times \mathbb{R}^{d} \times \mathbb{R}
    \end{equation*}
    be the ambient space in which $\mathcal{M}_{K}$ lives.
    Its dimension is given by
    \begin{equation*}
        m = \dim \mathcal{V} = \frac{d(d+1)}{2} + d +1.
    \end{equation*}
    By \cref{thm: reparam} and \cref{cor: reparam LLC}, it is sufficient to show that the differential
    \begin{equation*}
        d \Phi(\theta) \vert_{\theta=\theta^*}: T_{\theta^*}\Theta \to \mathcal{V}
    \end{equation*}
    is surjective.
    Proving surjectivity, the model $f_{\theta}$ is locally regular in the identifiable coordinates $(Q,r,s)$ and therefore
    \begin{equation*}
        \lambda = \frac{m}{2}.
    \end{equation*}

    Since $\textrm{Rank}(Q^{*})=d$, the active vectors $\{w_{j}^*: v_{j}^* \not =0\}$ span $\mathbb{R}^{d}$.
    Hence there exists an index set $I\subset\{1,\dots,K\}$ with $|I|=d$ such that $v^{*}_{i} \not = 0$ for $i\in I$ and
    \begin{equation*}
        W_{I}^{*}=[w_{i}^{*}]_{i\in I}\in \mathbb{R}^{d \times d}
    \end{equation*}
    is invertible.
    Without loss of generality, take $I=\{1,2,\dots, d\}$ and set $V_{I}^{*}= \textrm{diag}(v^{*}_{1},\cdots,v^{*}_{d})$ and $b_{I}^*:=(b_{1}^{*}, \dots, b_{d}^{*})$.

    Now, fix an arbitrary deformation $(\Delta Q, \Delta r, \Delta s) \in \mathcal{V}$.
    Our aim is to show that it can be realised by an arbitrary perturbation $(\Delta W, \Delta V, \Delta b, \Delta c)$ at the point $\theta^*$.

    \textbf{Step 1.} To realise $\Delta Q$, set $\Delta v = 0$ and $\Delta w_{j} =0$ for $j \not \in I$.
    Then the linearisation of $Q$ is
    \begin{equation*}
        \Delta Q = (\Delta W_{I})V_{I}^*{W_{I}^*}^{\top} + W_{I}^*V_{I}^*(\Delta W_{I}^{\top}).
    \end{equation*}
    Define
    \begin{equation*}
        X:= (W_{I}^*)^{-1} \Delta W_{I}, \quad R:= (W_{I}^*)^{-1} \Delta Q ({W_{I}^*}^{\top})^{-1}.
    \end{equation*}
    Since $\Delta Q$ is symmetric, $R$ is symmetric, and the equation becomes
    \begin{equation*}
        X V_{I}^* + V_{I}^* X ^{\top} = R.
    \end{equation*}
    Since $v_{i}^* \not = 0$ for $i \in I$, this system is solvable.
    For example, set
    \begin{equation*}
        X_{ii} = \frac{R_{ii}}{2v_{i}^*}, 
        \quad 
        X_{ij} = \frac{R_{ij}}{2v_{j}^*},
        \quad
        X_{ji} = \frac{R_{ij}}{2 v_{i}^*} \quad (i < j).
    \end{equation*}
    Hence $\Delta W_{I} = W_{I}^* X$ realises any prescribed $\Delta Q$.

    \textbf{Step 2.} With $\Delta v =0$ and $\Delta W$ fixed as in step 1., the linearisation of $r$ is given by
    \begin{equation*}
        \Delta r = (\Delta W_{I})V_{I}^*b_{I}^* + W_{I}^*V_{I}^* \Delta b_{I}.
    \end{equation*}
    Equivalently,
    \begin{equation*}
        W_{I}^*V_{I}^* \Delta b_{I} = \Delta r - r_{0}, \quad r_{0}:= (\Delta W_{I})V_{I}^*b_{I}^*.
    \end{equation*}
    Since $W_{I}^*$ and $V_{I}^*$ are invertible, there exists a unique solution
    \begin{equation*}
        \Delta b_{I} = (V_{I}^*)^{-1} (W_{I}^*)^{-1} \left( \Delta r - r_{0} \right).
    \end{equation*}
    Hence any $\Delta r$ is realised.

    \textbf{Step 3.} Finally, with $\Delta v = 0$, the linearisation of $s$ is
    \begin{equation*}
        \Delta s = 2b_{I}^* V_{I}^* \Delta b_{I} + \Delta c.
    \end{equation*}
    After fixing $\Delta b_{I}$ in step 2., choose
    \begin{equation*}
        \Delta c = \Delta s - 2 {b_{I}^*}^{\top} V_{I}^* \Delta b_{I}.
    \end{equation*}
    Hence, any $\Delta s$ is realised.

    Therefore, any element of $\mathcal{V}$ lies in the image of $d \Phi(\theta) \vert_{\theta = \theta^*}$, and so the Jacobian of $\Phi$ is locally surjective at $\theta^*$.
    By \cref{thm: reparam} and \cref{cor: reparam LLC}, the LLC is one half the identifiable dimension:
    \begin{equation*}
        \lambda = \frac{1}{2} \left( \frac{d(d+1)}{2} +d +1 \right) = \frac{(d+1)(d+2)}{4}.
    \end{equation*}
\end{proof}

\paragraph{Under-Parametrised Regime} 

\begin{theorem}\label{thm: LLC QNN regressor under-param}
    Let $d \in \mathbb{N}$ and $K\in \{1, \dots, d-1\}$, and let $f_{\theta}$ be the single-output quadratic network introduced above, trained with MSE loss.
    Suppose that $\theta^*=(W^*, V^*, b^*, c^*)$ is a true solution such that $v_{j}^* \not = 0$ for all $j$, $\{w_{j}^*\}_{j=1}^K$ are linearly independent, and $v_{i}^* + v_{j^*} \not = 0$ for $i \not = j$.
    Then the local learning coefficient at $\theta^*$ is
    \begin{equation*}
        \lambda = \frac{1}{2}\left(\frac{K(2d -K + 1)}{2} + K + 1 \right).
    \end{equation*}
\end{theorem}

\begin{proof}
    By \cref{thm: reparam} and \cref{cor: reparam LLC}, it is again sufficient to show that the differential
    \begin{equation*}
        d \Phi(\theta) \vert_{\theta = \theta^*} : T_{\theta^*} \Theta \to T_{\Phi(\theta^*)} \mathcal{M}_{K}
    \end{equation*}
    is surjective and where $\Phi(\theta^*)$ is a smooth point on $\mathcal{M}_{K}$.
    Since $\{w_{j}\}_{j=1}^K$ are assumed linearly independent and $v_{j}^* \not = 0$ for all $j$, we have that $\operatorname{Rank} (Q^*) = K$, so $\Phi(\theta^*) = (Q^*,r^*,s^*)$ is a smooth point on the manifold $\mathcal{M}_{K}$.

    Write
    \begin{equation*}
        W^*:= [w_{1}^*, \dots, w_{K}^*] \in \mathbb{R}^{d \times K}, \quad V^* := \textrm{diag} (v_{1}^*, \dots, v_{K}^*),
    \end{equation*}
    and let
    \begin{equation*}
        P := W^* ({W^*}^{\top} W^*)^{-1} {W^*}^{\top}
    \end{equation*}
    be the orthogonal projector onto $\textrm{col}(W^*)$.
    Set
    \begin{equation*}
        \alpha := v^* \odot b^* \in \mathbb{R}^K,
    \end{equation*}
    where $\odot$ is the element-wise vector multiplication operator.

    The tangent space at $Q^*$ of the rank-$K$ symmetric locus has dimension
    \begin{equation*}
        \dim T_{Q^*}\left(\operatorname{Sym}_{d}^{(K)}\right) = \frac{K(2d - K + 1)}{2},
    \end{equation*}
    and every tangent perturbation $\Delta Q$ can be written as
    \begin{equation}\label{eq:Q-tangent}
        \Delta Q = W^* S\,W^{*\top}+W^* V^* U^\top + U V^* W^{*\top}, 
        \quad \text{ for } 
        U\in\mathbb R^{d\times K}, 
        \quad S\in\mathrm{Sym}_K.
    \end{equation} 
    Since the component of $U$ that is parallel to $\textrm{col}(W^*)$ can be absorbed into $S$, we may without loss of generality choose
    \begin{equation*}
        {W^*}^{\top} U =0.
    \end{equation*}
    Differentiating the constraint $r\in \textrm{col}(W^*)$ at $(Q^*, r^*)$ then gives
    \begin{equation*}
        \Delta r = U \alpha + W^* \gamma
    \end{equation*}
    for some $\gamma \in \mathbb{R}^{K}$.
    Hence
    \begin{equation*}
        (I-P) \Delta r = U \alpha,
    \end{equation*}
    while the parallel component $P \Delta r = W^* \gamma \in \textrm{col} (W^*)$ is free.
    Therefore
    \begin{equation*}
        \dim T_{\Phi(\theta^*)} \mathcal{M}_{K} = \frac{K(2d - K +1)}{2} + K + 1.
    \end{equation*}

    It remains to show surjectivity of the Jacobian onto this tangent space.

    \textbf{Step 1.} Fix an arbitrary tangent vector $(\Delta Q, \Delta r, \Delta s) \in T_{\Phi(\theta*)}\mathcal{M}_{K}$.
    Choosing $U$ and and $S$ representing $\Delta Q$ as above, set
    \begin{equation*}
        \Delta W := U + W^* B, \quad \Delta V := \textrm{diag} (\Delta v_{1}, \dots, \Delta v_{K}) 
    \end{equation*}
    with $B \in \mathbb{R}^{K \times K}$ and $\Delta v \in \mathbb{R}^K$ to be determined.
    Then
    \begin{equation*}
        \sum_{j=1}^{K} \Delta Q_j = W^*(B V^*+V^* B^\top + \Delta V)W^{*\top} + W^*V^* U^\top + U V^* W^{*\top}.
    \end{equation*}
    Hence it is sufficient to solve
    \begin{equation*}
        BV^* + V^* B^{\top} + \Delta V = S.
    \end{equation*}

    Choosing $B$ symmetric, for $i \not = j$ we may set
    \begin{equation*}
        B_{ij} = \frac{S_{ij}}{v_{i}^* + v_{j}^*},
    \end{equation*}
    which is well-defined under the assumption $v_{i}^* + v_{j}^* \not = 0$, and absorb the diagonal into $\Delta V$ by taking
    \begin{equation*}
        \Delta v_{i} = S_{ii} - 2 v_{i}^* B_{ii}.
    \end{equation*}
    Hence any admissible $\Delta Q$ is realised.
    
    \textbf{Step 2.} Next,
    \begin{equation*}
        \Delta r = (\Delta W)\alpha + W^*\big((\Delta v)\odot b^* + v^*\odot \Delta b\big) = U\alpha + W^*\big(B\alpha + (\Delta v)\odot b^* + v^*\odot \Delta b\big).
    \end{equation*}
    Hence, the orthogonal component is automatically
    \begin{equation*}
        (I-P) \Delta r = U \alpha,
    \end{equation*}
    as required, while the parallel component can be matched arbitrarily:
    if $P \Delta r =  W^* \gamma$ for some $\gamma \in \mathbb{R}^{K}$, we solve
    \begin{equation*}
        \gamma = B \alpha + (\Delta v) \odot b^* + v^* \odot \Delta b
    \end{equation*}
    for $(\Delta v, \Delta b)$.
    Since there are $2K$ unknowns and only $K$ equations, this is always possible.

    \textbf{Step 3.} Finally, once $(\Delta v, \Delta b)$ are fixed, choose $\Delta c$ so that
    \begin{equation*}
        \Delta s=\sum_j\big((\Delta v_j)b_j^{*2}+2 v_j^* b_j^*(\Delta b_j)\big)+\Delta c.
    \end{equation*}

    Therefore, every tangent vector in $T_{\Phi(\theta^*)} \mathcal{M}_{K}$ is realised by a parameter perturbation, so
    \begin{equation*}
        \textrm{Im} J \Phi(\theta) \vert_{\theta = \theta*} = T_{\Phi(\theta^*)} \mathcal{M}_{K}.
    \end{equation*}
    By \cref{thm: reparam} and \cref{cor: reparam LLC},
    \begin{equation*}
        \lambda = \frac{1}{2} \dim T_{\Phi(\theta^*)} \mathcal{M}_{K} = \frac{1}{2} \left( \frac{K(2d - K +1)}{2} + K + 1 \right).
    \end{equation*}
\end{proof}

\begin{remark}\label{remark: LLC regressor extension to C^2}
    Although \cref{thm: LLC QNN regressor over-param,thm: LLC QNN regressor under-param} are stated for MSE loss for the sake of coherency, the dimension count is geometric and does not depend on the specific choice of squared loss.
    More generally, for any loss $\ell(y, f)$ that is $C^{2}$ in $f$, if the population excess risk admits a non-degenerate quadratic expansion in the identifiable coordinates $(Q,r,s)$ at $\theta^*$, with curvature form positive definite on $\textrm{Im} J \Phi(\theta^*)$, then the same tangent-space computation shows that the LLC is again half of the local identifiable dimension.
    Equivalently, if
    \begin{equation*}
        \mathcal{K}(\theta)=\frac{1}{2}(\Phi(\theta)-\Phi(\theta^*))^\top \mathcal I\,(\Phi(\theta)-\Phi(\theta^*))+o(\|\Phi(\theta)-\Phi(\theta^*))\|^2),
    \end{equation*}
    for some curvature operator $\mathcal{I}$ that is positive definite on the tangent directions identified above, then the same LLC formula follows.
\end{remark}

\begin{example}
    This example illustrates the extension in \cref{remark: LLC regressor extension to C^2} for $d=4$ and $K=2$.
    Let $X \sim \textrm{Unif}(S^3)$, let $\ell(y,f)$ be $C^2$ in $f$ and assume the curvature weight $w(x)$ is bounded above and below away from zero on $S^3$, i.e. $0 < c_{-} \leq w(x) \leq c_{+} < \infty$.
    Define the Fisher-Gram form
    \begin{equation*}
        \mathcal Q(\Delta,\Delta'):=\mathbb E\!\big[w(X)\,g_\Delta(X)\,g_{\Delta'}(X)\big],
        \quad
        g_{\Delta}(x):=x^\top \Delta Q\, x + 2\,\Delta r^\top x + \Delta s,
    \end{equation*}
    with $\Delta = (\Delta Q, \Delta r, \Delta s)$, associated with the curvature form.

    At $\Phi(\theta^*)=(Q^*, r^*, s^*)$ with $\operatorname{Rank} (Q^*)= 2$, \cref{thm: LLC QNN regressor under-param} gives
    \begin{equation*}
        \dim T_{\Phi(\theta^*)} \mathcal{M}_{2} = \frac{2(2 \cdot 4 - 2 +1)}{2} +2 + 1 = 10.
    \end{equation*}
    Thus, it remains only to verify that $\mathcal{Q}$ is positive definite on $T_{\Phi(\theta^*)} \mathcal{M}_{2}$.

    For $W^*=[e_{1}, e_{2}]$, every tangent perturbation $\Delta Q$ has zero bottom-right $2 \times 2$ block.
    On the other hand, if
    \begin{equation*}
        x^{\top} B x + 2 b^{\top} x + c = 0, \quad \forall x \in S^3,
    \end{equation*}
    then necessarily $b=0$ and $B=-c I_{4}$.
    Hence, every element of $\textrm{Ker} (\mathcal{Q})$ is proportional to $(I_{4}, 0, -1)$.
    However, since $I_{4}$ does not have zero bottom-right block, the kernel direction $(I_{4}, 0, -1)$ does not belong to  $T_{\Phi(\theta^*)} \mathcal{M}_{2}$ unless it is $0$.
    Therefore
    \begin{equation*}
        \textrm{ker} (\mathcal{Q}) \cap T_{\Phi(\theta^*)} \mathcal{M}_{2} = \{0\},
    \end{equation*}
    so $\mathcal{Q}$ is positive definite on $T_{\Phi(\theta^*)} \mathcal{M}_{2}$.

    It follows that the local excess risk is non-degenerate on all $10$ identifiable tangent directions.
    Therefore, the LLC is one half of the tangent dimension:
    \begin{equation*}
        \lambda = \frac{10}{2} = 5.
    \end{equation*}
\end{example}

\subsection{$p$-Output Regression Models}

Our previous theoretical results concern quadratic neural networks with a single scalar output. 
However, in many practical settings -- for example, modular arithmetic with one-hot encoded targets -- the model has $p$-outputs and is trained either as a stack of $p$ regressors with MSE loss or as a $p$-class classifier. 
In this appendix we consider the MSE setting and extend the quadratic-network analysis to $p$-output models.

To align with the modular addition setting studied in the main text, we consider the bias-free quadratic network
\begin{equation*}
    f_\theta(x)=V(W^\top x)^2,
\end{equation*}
with parameters $\theta=(W,V)$, where $W\in\mathbb R^{d\times K}$ and $V\in\mathbb R^{p\times K}$. 
Writing $w_j\in\mathcal{W}:=\mathbb{R}^d$ for the $j$-th column of $W$ and $v_{:j}\in\mathcal{U} := \mathbb R^p$ for the $j$-th column of $V$, the $k$-th output can be written as
\begin{equation*}
    f_{\theta,k}(x)
    =
    \sum_{j=1}^K v_{kj}(w_j^\top x)^2
    =
    x^\top\!\left(\sum_{j=1}^K v_{kj} w_j w_j^\top\right)x
    =
    x^\top Q_k x,
    \qquad k=1,\dots,p,    
\end{equation*}
where
\begin{equation*}
    Q_k:=\sum_{j=1}^K v_{kj} w_j w_j^\top \in \operatorname{Sym}(\mathbb R^{d\times d}).
\end{equation*}

It is convenient to regard this model as an algebraic family of $p$-output quadratic regression maps. 
In particular,
we can identify the space of real symmetric $d \times d$ matrices $\operatorname{Sym}(\mathbb{R}^{d \times d})$ with the abstract  space of symmetric tensors $\operatorname{Sym}^{2}(W)$ whose elements are $w \otimes w$.
Then the function space realised by the quadratic model class $f_{\theta}$ with architecture triple $(d,p,K)$ is the space of semi-symmetric tensors
\begin{equation*}
    \mathcal Y:=\mathcal U\otimes \operatorname{Sym}^2(\mathcal W)
    \cong \operatorname{Sym}(\mathbb R^{d\times d})^p,   
\end{equation*}
where
\begin{equation*}
    \dim \mathcal{Y} = pD, \qquad D:= \dim \operatorname{Sym}^{2}(\mathcal{W}) = \frac{d(d+1)}{2}.
\end{equation*}

Furthermore, the network induces the parameter map
\begin{equation}\label{eq: param map}
    \Phi_p(\theta)
    =
    \sum_{j=1}^K v_{:j}\otimes (w_j w_j^\top)
    =
    (Q_1,\dots,Q_p)
    \in \mathcal Y.    
\end{equation}
Thus, only perturbations of $(W,V)$ that change $(Q_1,\dots,Q_p)$ can affect the realised function and hence the local geometry of the loss. 
Our goal is therefore to study the rank of the Jacobian of $\Phi_p$ as a function of the architectural triple $(d,p,K)$.

Unless otherwise stated, all results are understood under the standard regularity assumptions of \citet[Theorem~7.1]{watanabe2009slt}.

\paragraph{Constructing the differential map} Given a parameter configuration $\theta$, we consider the differential of~\Cref{eq: param map}, namely
\begin{equation}\label{eq: differential map}
    d \Phi_{p}\mid_{\theta} \, = \sum_{j=1}^{K} \left[ \Delta v_{:j} \otimes (w_{j}w_{j}^{\top}) + v_{:j}\left( \Delta w_{j} w_{j}^{\top} + w_{j} \Delta w_{j}^{\top} \right) \right].
\end{equation}
Equivalently, for each output head we have
\begin{equation*}
    \Delta Q_{k} = \sum_{j=1}^{K} \left[
    (\Delta v_{kj})(w_{j}w_{j}^{\top}) + v_{kj} \left( \Delta w_{j} w_{j}^{\top} + w_{j} \Delta w_{j}^{\top} \right)
    \right]
\end{equation*}
We note that this parameterisation is not \textit{locally} identifiable in the raw coordinates $\theta$ as each hidden unit exhibits the following scaling symmetry:
\begin{equation*}
    w_{j} \mapsto \alpha \, w_{j}, \quad v_{:j} \mapsto \alpha^{-2} v_{:j}.
\end{equation*}
This leaves the tensor $v_{:j} \otimes (w_{j} w_{j}^{\top})$ invariant.

In order to proceed with our analysis of the Jacobian, we need to quotient out these symmetries by re-parametrising the differential map~\Cref{eq: differential map}.
To this end, we fix a hidden unit $j$ and consider the following decomposition:
\begin{equation*}
    \Delta w_{j} = u_{j} + \alpha_{j} w_{j}, \quad u_{j} \perp w_{j}.
\end{equation*}
Then
\begin{equation*}
    \Delta w_{j} w_{j}^{\top} + w_{j} \Delta w_{j}^{\top} = u_{j} w_{j}^{\top} + w_{j}u_{j}^{\top} + 2 \alpha_{j} w_{j} w_{j}^{\top}.
\end{equation*}
Substituting into the differential~\Cref{eq: differential map}, we obtain
\begin{equation*}
    d \Phi_{p}\mid_{\theta} \, = \sum_{j=1}^{K} \left[ \left( \Delta v_{:j} + 2 \alpha_{j} v_{:j}\right) \otimes (w_{j}w_{j}^{\top}) + v_{:j}\left( u_{j} w_{j}^{\top} + w_{j} u_{j}^{\top} \right) \right].
\end{equation*}
Defining the symmetry-quotiented coefficient variation
\begin{equation*}
    \widehat{\Delta v_{:j}}:= \Delta v_{:j} + 2 \alpha_{j} v_{:j},
\end{equation*}
the quotiented differential becomes
\begin{equation}\label{eq: quot differential map}
    d \Phi_{p}^{\operatorname{quot.}}\mid_{\theta} \, = \sum_{j=1}^{K} \left[ \widehat{\Delta v_{:j}} \otimes (w_{j}w_{j}^{\top}) + v_{:j}\otimes\left( u_{j} w_{j}^{\top} + w_{j} u_{j}^{\top} \right) \right], \quad u_{j} \perp w_{j}.
\end{equation}
This is precisely the object of our interest and whose rank we would like to compute.

\paragraph{Single hidden unit rank contribution}

Given a hidden neuron $j$, we define the following:
\begin{equation*}
    M_{j}:= w_{j} w_{j}^{\top} \quad \textrm{and} \quad T_{j}:= \{u_{j}w_{j}^{\top} + w_{j}u_{j}^{\top}: u_{j} \perp w_{j}\}.
\end{equation*}
We then define two subspaces of the codomain $\mathcal{Y}$:
\begin{equation*}
    \mathcal{A}_{j}:= \mathcal{U} \otimes \operatorname{Span}(M_{j}) \quad \textrm{and} \quad \mathcal{B}_{j}:= \operatorname{Span}(v_{:j}) \otimes T_{j}.
\end{equation*}
Then $\mathcal{C}_{j}:= \mathcal{A}_{j} + \mathcal{B}_{j}$ is precisely the subspace contributed by the hidden neuron $j$ to the quotiented Jacobian in~\Cref{eq: quot differential map}.
Our aim is to determine the dimension of $\mathcal{C}_{j}$ considering the relation between $\mathcal{A}_j$ and $\mathcal{B}_j$.

\begin{proposition}
    Let $\mathcal{C}_{j}$ be the subspace of $\mathcal{Y}$ generated by the $j$ neuron.
    Then
    \begin{equation*}
        \dim \mathcal{C}_{j} = d +p-1, \quad j=1,\dots,K
    \end{equation*}
\end{proposition}

\begin{proof}
    We have that
    \begin{equation*}
    \dim \mathcal{C}_{j} := \dim \mathcal{A}_{j} + \dim \mathcal{B}_{j} - \dim (\mathcal{A}_{j} \cap \mathcal{B}_{j}).
    \end{equation*}
    \textbf{Claim 1: $\dim \mathcal{A}_{j} = p$.}
    Assume that $w_{j} \not = 0$ and $v_{:j}\not = 0$.
    Then $M_{j}=w_{j} w_{j}^{\top}$ is non zero and hence the linear map
    \begin{equation*}
        \mathcal{U} \to \mathcal{A}_{j}, \quad \beta \mapsto \beta \otimes M_{j},
    \end{equation*}
    where $\beta \in \mathcal{U}$, is injective (i.e. $\beta \otimes M_{j} =0 \implies \beta =0 $).
    Hence $\dim \mathcal{A}_{j} = p$.
    \\
    
    \noindent
    \textbf{Claim 2: $\dim \mathcal{B}_{j}=d-1$.}
    Let $\mathcal{W}_{j}^{\top}:=\{u: u\perp w_{j}\}$.
    Define the linear map
    \begin{equation*}
        \mathcal{W}_{j}^{\top} \to \mathcal{B}_{j}, \quad u\mapsto v_{:j} \otimes (uw_{j}^{\top} + w_{j}u^{\top}).
    \end{equation*}
    Since $v_{:j} \not = 0$, it is sufficient to show that
    \begin{equation*}
        uw_{j}^{\top} + w_{j}u^{\top} =0 \implies u=0.
    \end{equation*}
    Begin by multiplying the left-hand side of the above by $w_{j}$:
    \begin{equation*}
        (uw_{j}^{\top} + w_{j}u^{\top})w_{j} = u \|w_{j}\|^{2} =0.
    \end{equation*}
    Since $w_{j}\not =0$, the above holds iff $u=0$.
    Hence, the map is injective and we have
    \begin{equation*}
        \dim \mathcal{B}_{j} = \dim \mathcal{W}_{j}^{\top} = d-1.
    \end{equation*}
    \\
    \noindent
    \textbf{Claim 3: $\mathcal{A}_{j} \cap \mathcal{B}_{j}=\{0\}$.}
    Let $y \in \mathcal{A}_{j} \cap \mathcal{B}_{j}$.
    Then
    \begin{equation*}
        y = \beta \otimes M_{j} = v_{:j} \otimes N
    \end{equation*}
    for some $\beta \in \mathcal{U}$ and some
    \begin{equation*}
        N= uw_{j}^{\top} + w_{j}u^{\top}\in T_{j}, \quad u \perp w_{j}.
    \end{equation*}
    Consider the Frobenius inner product with $M_{j}= w_{j} w_{j}^{\top}$.
    We have
    \begin{equation*}
        \langle N,M_{j} \rangle_{F} = \langle uw_{j}^{\top} + w_{j}u^{\top}, w_{j}w_{j}^{\top} \rangle_{F} =0
    \end{equation*}
    since $u \perp w_{j}$.
    However, if $\beta \otimes M_{j} = v_{:j} \otimes N$, then for any output coordinate where $v_{:j} \not =0$, the corresponding component of $y$ is simultaneously a scalar multiple of $M_{j}$ and of $N$.
    So, $N$ must also be proportional to $M_{j}$.
    But we have already shown that $M_{j}$ and $N$ are orthogonal to each other, which implies that $N=0$.
    Then $y=0$ and therefore $\mathcal{A}_{j} \cap \mathcal{B}_{j} = \{0\}$.
\end{proof}

\paragraph{Geometric Interpretation}

Consider the set of single hidden neuron contributions
\begin{equation*}
    \widehat{\mathcal{X}} := \{v \otimes (ww^{\top}) : v\in \mathcal{U}, w\in \mathcal{W}\} \subset \mathcal{Y}.
\end{equation*}
This set is an \textit{affine cone}, and an element $x \in \widehat{\mathcal{X}}$ is exactly one rank-one partially symmetric tensor $v\otimes (ww^{\top})$.
The set is a \textit{cone} because it is a set that is closed under scalar multiplication, and it is \textit{affine} because it is viewed inside the ordinary vector space $\mathcal{Y}$.
We can define its associated \textit{projective variety} $\mathcal{X}$ by identifying points in $x$ up to non-zero scalar multiplication:
\begin{equation}\label{eq: projection of X}
    \mathcal{X}=\{\left[v \otimes (ww^{\top})\right]\} \subset \mathbb{P}(\mathcal{Y}),
\end{equation}
where $[\cdot]$ is read as ``up to non-zero scalar".
Then the affine cone is simply the preimage of the projective variety under the map
\begin{equation*}
    \mathcal{Y} \setminus \{0\} \to \mathbb{P}(\mathcal{Y}), \quad y \mapsto [y].
\end{equation*}
We can interpret the projective object $\mathcal{X}$ as remembering only the \textit{directions} of a rank-one atom, whereas $\widehat{\mathcal{X}}$ remembers the tensor itself.

We can rewrite the projection of $\widehat{\mathcal{X}}$ in the following way:
\begin{equation*}
    \mathcal{X} = \mathbb{P}(\mathcal U) \times \nu_{2}(\mathbb{P}(\mathcal W)) \subset \mathbb{P}(\mathcal Y).
\end{equation*}
This is a \textit{Segre-Veronese} variety, i.e. it consists of two ingredients: a linear $v$-part that corresponds to a \textit{Segre-type factor} and a quadratic $w$-part that corresponds to a \textit{Veronese-type} factor.
The \textit{quadratic Veronese embedding} $\nu_{2}$ simply sends a point to all degree-two monomials in its coordinates.
For example, if 
\begin{equation*}
    w=(w_{1}, \dots, w_{d}),
\end{equation*}
then the map $\nu_{2}$ sends it to
\begin{equation*}
    (w_{1}^{2}, w_{1}w_{2}, \dots, w_{i}w_{j}, \dots, w_{d}^{2})
\end{equation*}
but this contains precisely the same information as the rank-one symmetric matrix $ww^{\top}$.
Hence, it is simply the map
\begin{equation*}
    [w] \mapsto [ww^{\top}].
\end{equation*}

\begin{proposition}\label{prop: tangent space to a point xj}
    Given a non-zero point $x_{j} \in \widehat{\mathcal{X}}$, the tangent space to $\widehat{\mathcal{X}}$ at $x_{j}$ is exactly
    \begin{equation*}
        T_{x_{j}} \widehat{\mathcal{X}} = \{ \beta \otimes (w_{j} w_{j}^{\top}) + v_{:j} \otimes (uw_{j}^{\top} + w_{j}u^{\top}): \beta \in U, u \perp w_{j}\}.
    \end{equation*}
    That is $T_{x_{j}} \widehat{\mathcal{X}} = \mathcal{C}_{j}$.
\end{proposition}

\begin{proof}
    Let $x_{j} = v_{:j} \otimes(ww^{\top})$.
    To compute the tangent space at $x_{j}$, we parametrise the set $\widehat{\mathcal X}$ by the map $g$ whose image is precisely $\widehat{\mathcal X}$.
    Hence, we define
    \begin{equation*}
        f: \mathcal U \times \mathcal W \to \mathcal Y, \quad f(v,w)= v \otimes (ww^{\top})
    \end{equation*}
    such that the derivate of $f$ at $(v_{:j},w_{j})$ is the tangent space at $x_{j}$.
    For clarity, we can compute the derivate by considering the perturbation
    \begin{equation*}
        v(t) = v_{:j} + t\beta, \quad w(t) = w_{j} + t \Delta w.
    \end{equation*}
    Then
    \begin{equation*}
        f(v(t), w(t)) = (v_{:j} + t\beta) \otimes \left( (w_{j} + t \Delta w) (w_{j} + t \Delta w)^{\top} \right).
    \end{equation*}
    Expanding and collecting linear $t$-terms
    \begin{equation*}
        f(v(t), w(t)) = v_{:j} \otimes (w_{j}w_{j}^{\top}) + t \left[ \beta \otimes (w_{j}w_{j}^{\top}) + v_{:j} \otimes (\Delta w_{j}w_{j}^{\top} + w_{j} \Delta w^{\top}) \right] + O(t^2).
    \end{equation*}
    Hence
    \begin{equation*}
        d f_{(v_{:j},w_{j})}(\beta, \Delta w) = \beta \otimes (w_{j}w_{j}^{\top}) + v_{:j} \otimes (\Delta w_{j}w_{j}^{\top} + w_{j} \Delta w^{\top})
    \end{equation*}
    Again, by considering the decomposition
    \begin{equation*}
        \Delta w = u + \alpha w_{j}, \quad u \perp w_{j}
    \end{equation*}
    we can obtain
    \begin{equation*}
        d f_{(v_{:j},w_{j})}(\beta, \Delta w) = (\beta + 2 \alpha v_{:j}) \otimes (w_{j}w_{j}^{\top}) + v_{:j} \otimes (uw_{j}^{\top} + w_{j} u^{\top}).    
    \end{equation*}
    However, the term $\alpha w_{j}$ does not introduce any new directions in the tangent space since it is just a coefficient of $w_{j}w_{j}^{\top}$.
    So, it can be absorbed into $\beta$ as it is arbitrary.
    Therefore, the image of the differential is precisely $\mathcal{C}_{j}$.
\end{proof}

\noindent
Since $x_{j}:= v_{:j} \otimes (w_{j}w_{j}^\top) \in \widehat{\mathcal{X}}$ is the output of each atom $j=1, \dots K$, we have
\begin{equation*}
    \Phi_{p}(\theta) = x_{1} + \dots + x_{K}.
\end{equation*}

\begin{proposition}
    For any $\theta \in \Theta$ where $w_{j} \not = 0$ and $v_{:j} \not = 0$, we have that
    \begin{equation*}
        \operatorname{Im} J \Phi_{p}^{\operatorname{quot}}(\theta) = T_{x_{1}}\widehat{\mathcal X} + \dots + T_{x_{K}}\widehat{\mathcal X}.
    \end{equation*}
\end{proposition}

\begin{proof}
    Given the quotiented differential map in~\Cref{eq: quot differential map}
    \begin{equation*}
    d \Phi_{p}^{\operatorname{quot.}}\mid_{\theta} \, = \sum_{j=1}^{K} \left[ \widehat{\Delta v_{:j}} \otimes (w_{j}w_{j}^{\top}) + v_{:j}\left( u_{j} w_{j}^{\top} + w_{j} u_{j}^{\top} \right) \right], \quad u_{j} \perp w_{j},        
    \end{equation*}
    we have that the $j$-th summand ranges over $T_{x_{j}}\widehat{\mathcal X}$ by~\Cref{prop: tangent space to a point xj}.
    Hence, the image of the full quotient Jacobian is
    \begin{equation*}
        \operatorname{Im} J \Phi_{p}^{\textrm{quot}}(\theta)=\sum_{j=1}^{K} T_{x_{j}}\widehat{\mathcal X}.
    \end{equation*}
\end{proof}

\paragraph{Why secant-varieties are relevant}
As previously discussed, each hidden neuron realises a semi-symmetric tensor of the form $v \otimes (ww^{\top})$ and is a point in the affine cone $\widehat{\mathcal X}$.
A width-$K$ network therefore produces a sum of $K$ such one-atom tensors
\begin{equation*}
    \Phi_{p}(\theta) = \sum_{j=1}^{K} v_{:j} \otimes (w_{j} w_{j}^{\top}).
\end{equation*}
Hence, a width-$K$ model class is precisely the set of all sums of $K$ points in $\widehat{\mathcal X}$.

To study this sum-set geometrically, it is convenient for us to pass to the projective space where tensors differing only by an overall non-zero scalar are identified.
This projectivisation of $\widehat{\mathcal X}$ is the Segre-Varonese variety
\begin{equation*}
    \mathcal X \subset \mathbb{P}(\mathcal Y)
\end{equation*}
whose points are precisely the projective classes of single hidden neuron atoms $[v \otimes (ww^{\top})].$
The corresponding width-$K$ model class is then encoded by \textit{$K$-th secant variety of $\mathcal X$}, defined by:
\begin{equation*}
    \sigma_{K}(\mathcal X) := \overline{\bigcup_{x_{1}, \dots, x_{K}\in \mathcal X} \langle x_{1}, \dots, x_{k} \rangle}.
\end{equation*}
Here $\langle x_{1}, \dots, x_{k} \rangle$ denotes the \textit{projective linear span} of the points $x_{1}, \dots, x_{K}\in \mathcal X$, i.e. all projective classes of linear combinations $[c_{1}x_{1} + \dots + c_{K}x_{K}]$.
The closure is included so that \textit{limiting sums} are also captured.
In this sense, the secant variety is the natural projective geometric model of width-$K$ outputs that we would like to study.

The relevance of this construction is that it allows us to relate the rank of the quotient Jacobian $\Phi_{p}$ to the dimension of the corresponding secant variety.
We can obtain this relationship in two steps.

First, let
\begin{equation*}
    x_{j}:= v_{:j} \otimes (w_{j}w_{j}^{\top}) \in \widehat{\mathcal X}, \quad y:= x_{1} + \dots x_{K} = \Phi_{p}(\theta).
\end{equation*}
By~\cref{prop: tangent space to a point xj}, we know that the $j$-th hidden neuron contributes exactly the tangent space $T_{x_{j}}\widehat{\mathcal X}$ to the quotient Jacobian.
Hence, the image of the quotient Jacobian is precisely
\begin{equation*}
    \operatorname{Im} J \Phi_{p}^{\textrm{quot}}(\theta) = T_{x_{1}}\widehat{\mathcal X} + \dots + T_{x_{K}}\widehat{\mathcal X}.
\end{equation*}

Second, \textit{Terracini's lemma} identifies this same tangent space sum with the tangent space to the secant model class at a general point.
More precisely:
\begin{theorem}[Terracini's lemma (informal)]
    If $x_{1}, \dots, x_{K} \in \widehat{\mathcal X}$ are in a general position and $y=x_{1} + \dots + x_{K}$ is a generic smooth point of the affine secant cone over $\sigma_{K}(\mathcal X)$, then
    \begin{equation*}
        T_{y} \widehat{\sigma_{K}(\mathcal X)} = T_{x_{1}}\widehat{\mathcal X} + \dots + T_{x_{K}}\widehat{\mathcal X}.
    \end{equation*}
    Equivalently, the tangent space at the generic point $y$ of the width-$K$ model class is simply the sum of the tangent spaces of the constituent one-atom varieties.
\end{theorem}
Hence, we obtain the result that
\begin{equation*}
    \operatorname{Im} J \Phi_{p}^{\operatorname{quot}}(\theta) = T_{y} \widehat{\sigma_{K}(\mathcal X)},
\end{equation*}
and therefore
\begin{equation}\label{eq: secant dimension}
    \operatorname{Rank} J \Phi_{p}^{\operatorname{quot}}(\theta) = \dim T_{y} \widehat{\sigma_{K}(\mathcal X)} = \dim \widehat{\sigma_{K}(\mathcal X)},
\end{equation}
where we have used the assumption that $y$ is a generic smooth point.
Equivalently, the generic rank of the quotient Jacobian is reduced to the problem of computing the dimension of the secant variety associated with the one-atom model class.

\paragraph{LLC results}

\Cref{eq: secant dimension} provides us with the identifiable dimension of the local parameter space at a generic $\theta$.
By applying the reparametrisation theorem (~\Cref{thm: reparam}) that links the LLC to half the number of local identifiable dimensions under MSE, we obtain the following generic formula.

\begin{theorem}[Generic secant-dimension formula for LLC]\label{thm: generic secant-dim LLC formula}
    Let $f_{\theta}$ be a quadratic network with architecture triple $(d,p,K)$, parametrised by $\theta =(V,W)$, that realises the parameter map:
    \begin{equation*}
            \Phi_p(\theta)
    =
    \sum_{j=1}^K v_{:j}\otimes (w_j w_j^\top)
    \in \mathcal Y, \quad \mathcal{Y}:= \mathbb{R}^{p} \otimes \operatorname{Sym} (\mathbb{R}^{d \times d}).
    \end{equation*}
    Let
    \begin{equation*}
        \mathcal{X}= \{[v\otimes(ww^{\top})]: v\in \mathcal{U}, w\in \mathcal{W}\}
        \subset \mathbb{P}(\mathcal{Y})
    \end{equation*}
    denote the projective variety of one-hidden neuron atoms, and $\widehat{\sigma_{K}(\mathcal{X})}$ denote the affine cone over its $K$-th secant variety.
    Assume the following:
    \begin{enumerate}
        \item $\theta$ is a generic parameter point;
        \item $w_{j} \not = 0$ and $v_{:j} \not =0$ for all $j=1,\dots, K$;
        \item the atoms $x_{1},\dots, x_{K}$ are in a general position and $y=x_{1}+\dots + x_{K}$ is a generic smooth point on $\widehat{\sigma_{K}(\mathcal X)}$;
        \item the assumptions of the reparameterisation theorem for MSE (\Cref{thm: reparam}) hold at $y=\Phi_{p}(\theta)$.
    \end{enumerate}
    Then
    \begin{equation*}
        \operatorname{Rank} J \Phi_{p}^{\operatorname{quot}}(\theta) = \dim \widehat{\sigma_{K}(\mathcal X)},
    \end{equation*}
    and consequently
    \begin{equation*}
        \lambda = \frac{1}{2} \dim \widehat{\sigma_{K}(\mathcal X)}.
    \end{equation*}
\end{theorem}

So determining the LLC boils down to determining the expected dimension of the secant variety $\widehat{\sigma_{k}(\mathcal X)}$ at a generic smooth point $y$.
Consider the following: each hidden neuron contributes $d+p-1$ quotient directions while the ambient space has dimension
\begin{equation*}
    pD = p \frac{d(d+1)}{2}.
\end{equation*}
Hence the expected dimension at a generic point is given by
\begin{equation*}
    r_{\textrm{exp}} := \min \left( K(p+d-1), p \frac{d(d+1)}{2} \right).
\end{equation*}
We say that the architecture $(p,d,K)$ is \textit{non-defective} if
\begin{equation*}
    \dim \widehat{\sigma_{K}(\mathcal{X})} = r_{\textrm{exp}},
\end{equation*}
and defective otherwise.
Under this assumption, the LLC takes the explicit closed form.
\begin{corollary}\label{cor: generic LLC}
    Under the assumptions of~\Cref{thm: generic secant-dim LLC formula}, and assuming that the secant variety $\sigma_{K}(\mathcal X)$ is non-defective, then
    \begin{equation*}
        \operatorname{Rank} J \Phi_{p}^{\operatorname{quot}}(\theta) = \min \left( K(p+d-1), p \frac{d(d+1)}{2} \right),
    \end{equation*}
    and consequently
    \begin{equation*}
        \lambda = \frac{1}{2} \min \left( K(p+d-1), p \frac{d(d+1)}{2} \right).
    \end{equation*}    
\end{corollary}

\paragraph{Regime splitting}
Considering this projective geometry approach, two \textit{regimes} naturally emerge in our analysis according to the comparison between the expected quotient dimension $K(p+d-1)$ and the ambient dimension $pD$.

\begin{corollary}[Subabundant regime]
    Suppose that the assumptions of \Cref{cor: generic LLC} hold, where the architecture is non-defective and satisfies
    \begin{equation*}
        K(d+p-1) < p \frac{d(d+1)}{2}.
    \end{equation*}
    Then
    \begin{equation*}
        \operatorname{Rank} J \Phi_{p}^{\operatorname{quot}}(\theta) = K(p+d-1), \quad \lambda = \frac{K(p+d-1)}{2}.
    \end{equation*}
\end{corollary}

\begin{proof}
    In the subabundant regime, the expected secant dimension is precisely $K(d+p-1)$.
    Non-defectivity implies that the actual secant dimension is equal to the expected secant dimension.
    The result follows from the previous corollary.
\end{proof}

\begin{corollary}[Superabundant Regime]
    Suppose that the assumptions of \Cref{cor: generic LLC} hold, where the architecture is non-defective and satisfies
    \begin{equation*}
        K(d+p-1) \geq p \frac{d(d+1)}{2}.
    \end{equation*}
    Then
    \begin{equation*}
        \operatorname{Rank} J \Phi_{p}^{\operatorname{quot}}(\theta) = p\frac{d(d+1)}{2}, \quad \lambda = p\frac{d(d+1)}{4}.
    \end{equation*}    
\end{corollary}

\begin{proof}
    In the superabundant regime, the expected secant dimension saturates the ambient space dimension $pD$.
    Again, non-defectivity implies that the actual secant dimension equals $pD$, and the LLC formula follows immediately.
\end{proof}

\section{Additional Experiments} \label{app: additional experiments}
Here we include more experiments conducted to further support the claims made in the main part of the paper as well as presenting, for reproducibility purposes, the set-up used in the experiments.
\subsection{Experimental setup}\label{app: experiments/setup}
All of the specific set-ups can be found in the repository \url{https://anonymous.4open.science/r/geom_phase_transitions-DF59/}. Indeed, under the results folder, each run folder contains two csv files. The params.csv file contains a list of the all model hyperparameters, seeds, dataset information, etc. The loss\_data.csv file contains a summary of the training and validation losses, training and validation accuracies, and LLC values during training. 

As a baseline, we use dataset group size $p = 53$, $100{,}000$ training epochs with checkpoints every $100$ epochs, weight decay $10^{-5}$, training fraction $0.4$, batch size $128$, and random seed $0$.

We use SGLD with step size $\epsilon = 10^{-4}$, inverse temperature $\beta = 30/n$, localization strength $\gamma = 5$, $3$ independent chains, $100$ burn-in steps, and $600$ draws per chain. According to the sensitivity of the LLC estimator, we keep sampler hyperparameters fixed across all runs and interpret absolute LLC values up to an estimator-dependent scale and focus on predicted scaling relationships and within-estimator comparisons.

\subsection{LLC tracks the emergence of generalisation- varying hyperparameters}
In the main paper we show an example of a plot showing the training and validation loss, as well as the LLC curve, during training for a model trained for $p=53$, learning rate 0.0001, weight decay 0.00001, batch size 128 and hidden dimension 1024. Here we present similar plots to show that the relationships explained in section \ref{sec::experiments} are robust to variations in datasets, model hyperparameters and training parameters. 
\subsubsection{Varying the datasets: $p$}
    \begin{figure}[H]
    \centering
      \begin{subfigure}[t]{0.32\linewidth}
        \centering
        \includegraphics[width=\linewidth]{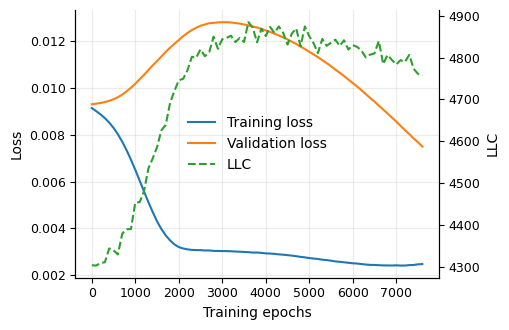}
        \caption{ $p=53$}
      \end{subfigure}\hfill
      \begin{subfigure}[t]{0.32\linewidth}
        \centering
        \includegraphics[width=\linewidth]{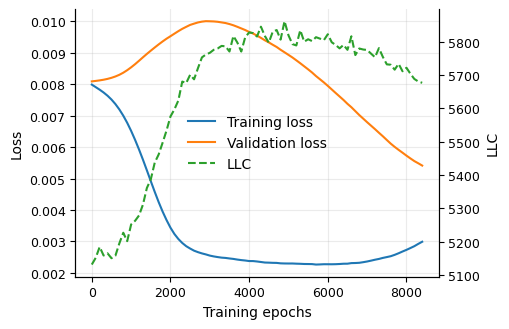}
        \caption{$p=61$}
      \end{subfigure}\hfill
      \begin{subfigure}[t]{0.32\linewidth}
        \centering
        \includegraphics[width=\linewidth]{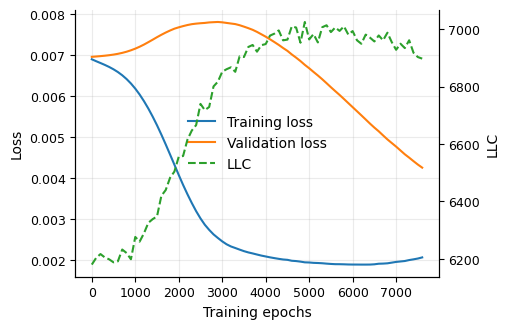}
        \caption{ $p=71$}
      \end{subfigure}
      \caption{Tracking LLC curves and the parallel evolution of generalisation for different values of $p$. }
    \end{figure}
\subsubsection{Varying the model: dimension of hidden layer}
    \begin{figure}[H]
    \centering
      \begin{subfigure}[t]{0.32\linewidth}
        \centering
        \includegraphics[width=\linewidth]{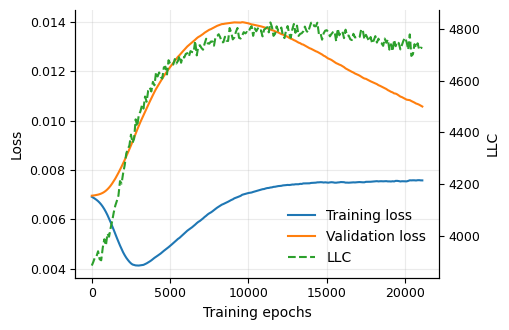}
        \caption{ Dimension of hidden layer $=600$}
      \end{subfigure}\hfill
      \begin{subfigure}[t]{0.32\linewidth}
        \centering
        \includegraphics[width=\linewidth]{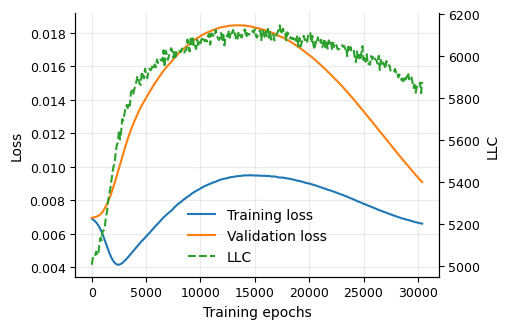}
        \caption{Dimension of hidden layer $=800$}
      \end{subfigure}\hfill
      \begin{subfigure}[t]{0.32\linewidth}
        \centering
        \includegraphics[width=\linewidth]{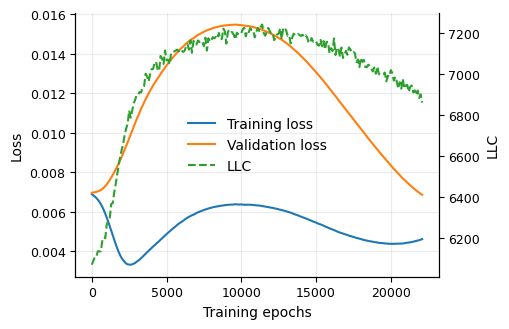}
        \caption{Dimension of hidden layer $=1000$}
      \end{subfigure}
      \caption{Tracking LLC curves and the parallel evolution of generalisation for different dimensions of the hidden layer. }
    \end{figure}

\subsubsection{Varying the training parameters: learning rate and weight decay}
    \begin{figure}[H]
    \centering
      \begin{subfigure}[t]{0.32\linewidth}
        \centering
        \includegraphics[width=\linewidth]{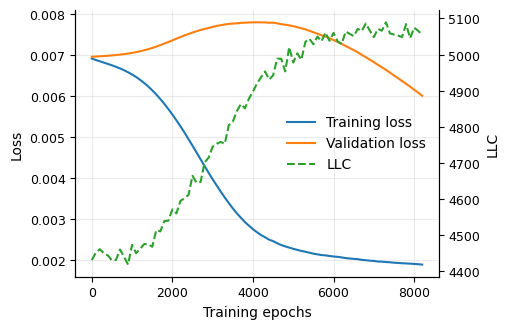}
        \caption{ Weight decay $=0.0001$}
      \end{subfigure}\hfill
      \begin{subfigure}[t]{0.32\linewidth}
        \centering
        \includegraphics[width=\linewidth]{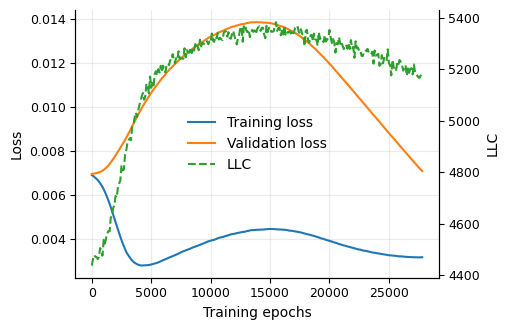}
        \caption{Weight decay $=0.00005$}
      \end{subfigure}\hfill
      \begin{subfigure}[t]{0.32\linewidth}
        \centering
        \includegraphics[width=\linewidth]{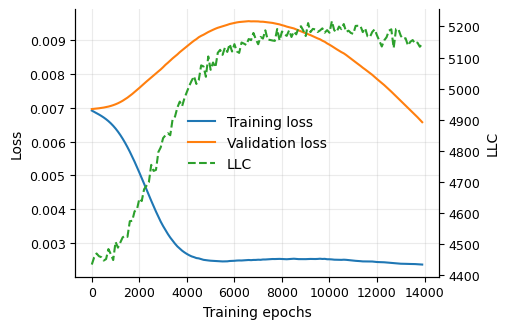}
        \caption{ Weight decay$=0.00001$}
      \end{subfigure}
      \caption{Tracking LLC curves and the parallel evolution of generalisation for different weight decays. }
    \end{figure}

    \begin{figure}[H]
    \centering
      \begin{subfigure}[t]{0.32\linewidth}
        \centering
        \includegraphics[width=\linewidth]{images/tracking/lr=0.0001_weight_decay=1e-05.png}
        \caption{ Learning rate $=0.0001$ }
      \end{subfigure}\hfill
      \begin{subfigure}[t]{0.32\linewidth}
        \centering
        \includegraphics[width=\linewidth]{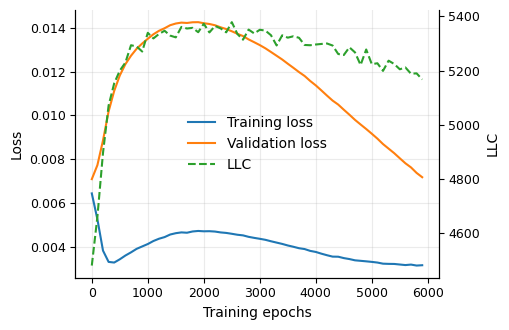}
        \caption{Learning rate $=0.001$}
      \end{subfigure}\hfill
      \begin{subfigure}[t]{0.32\linewidth}
        \centering
        \includegraphics[width=\linewidth]{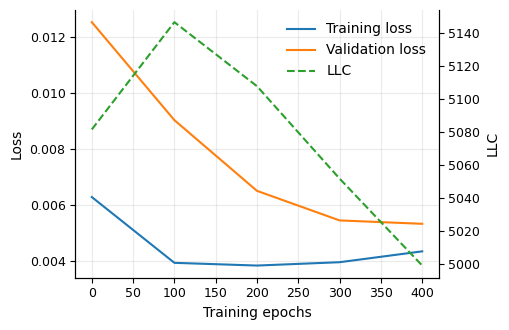}
        \caption{Learning rate $=0.01$}
      \end{subfigure}
      \caption{Tracking LLC curves and the parallel evolution of generalisation for different learning rates. }
    \end{figure}

\subsection{LLC explains the effect of learning rate on grokking severity}\label{appendix: subsection: grokking severity}

\begin{figure}[h]
    \centering
    \includegraphics[width=0.5\linewidth]{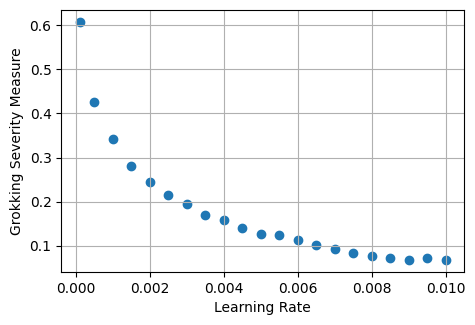}
    \caption{The same model is trained with different learning rates. The learning rate is plotted against its GSM.}
    \label{fig: lr_vs_grokking}
\end{figure}

To quantify the effect of optimisation hyperparameters on delayed generalisation, we introduce a grokking severity measure (GSM), which captures the cumulative gap between memorisation and generalisation, conditional on eventual successful generalisation. If $a_T(x),a_V(x)$ are accuracies of the model at training epoch $x\in [0,T]$ with respect to the training and validation datasets respectively, we define 
\begin{align*}
    \operatorname{GSM} = \frac{\mathbbm{1}_{\{a_V(T) \geq 0.95\}}}{T}\sum_{x=1}^T|a_T(x)-a_V(x)|.
\end{align*} 
Thus, GSM is large when a model memorises early but spends a long time before achieving strong validation performance, and vanishes for runs that do not eventually generalise.
Figure \ref{fig: lr_vs_grokking} shows that GSM decreases markedly as the learning rate increases. Similar trends for other hyperparameter settings are reported in Appendix \ref{appendix: lr_vs_maxLLC}. We interpret this trend through the lens of SLT. In Appendix \ref{appendix: lr_vs_maxLLC}, we show that larger learning rates are also associated with smaller peak LLC values during training. Taken together, these observations are consistent with an SLT-based picture in which larger learning rates bias optimisation away from sharper, higher-LLC low-loss regions, thereby reducing the time spent in a memorising-but-not-yet-generalising regime. We stress, however, that this is an interpretive connection rather than a direct derivation of SGD dynamics from SLT.

\subsection{LLC explains the effect of learning rate on grokking severity- further experiments} \label{appendix: lr_vs_maxLLC}
We repeated the experiment with varying values of weight decay to give some robustness to the analysis conducted. 
    \begin{figure}[H]
    \centering
      \begin{subfigure}[t]{0.32\linewidth}
        \centering
        \includegraphics[width=\linewidth]{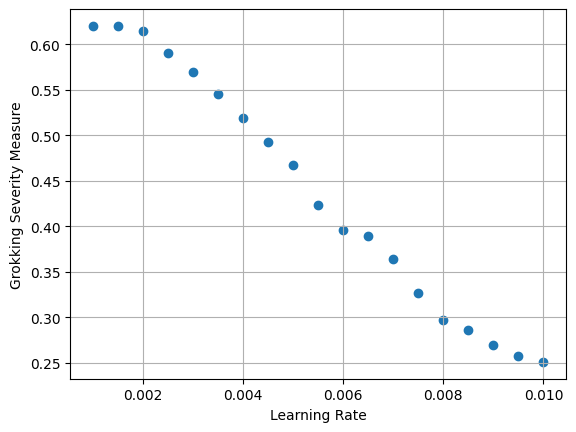}
        \caption{ Weight decay $=0$}
      \end{subfigure}\hfill
      \begin{subfigure}[t]{0.32\linewidth}
        \centering
        \includegraphics[width=\linewidth]{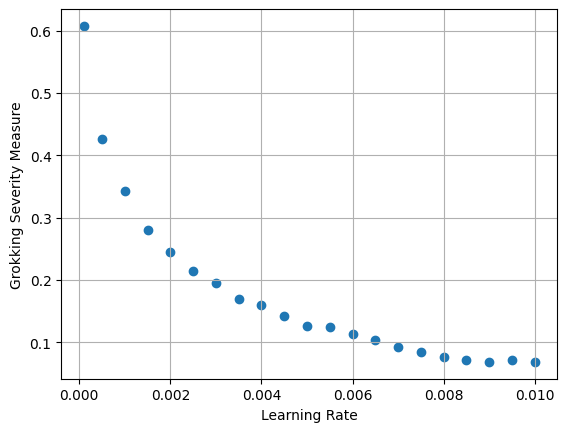}
        \caption{Weight decay $=0.0001$}
      \end{subfigure}\hfill
      \begin{subfigure}[t]{0.32\linewidth}
        \centering
        \includegraphics[width=\linewidth]{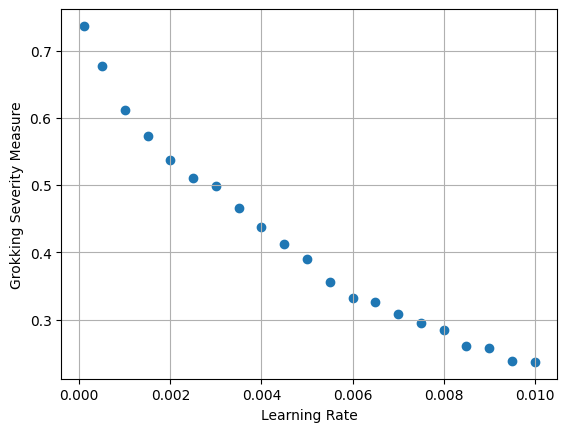}
        \caption{ Weight decay$=0.00001$}
      \end{subfigure}
      \caption{Relationship between the learning rate and the severity of grokking for different weight decays. }
      \end{figure}

We also examined the relationship between the varying learning rate and the maximum point of the LLC curve.
      
    \begin{figure}[H]
    \centering
      \begin{subfigure}[t]{0.32\linewidth}
        \centering
        \includegraphics[width=\linewidth]{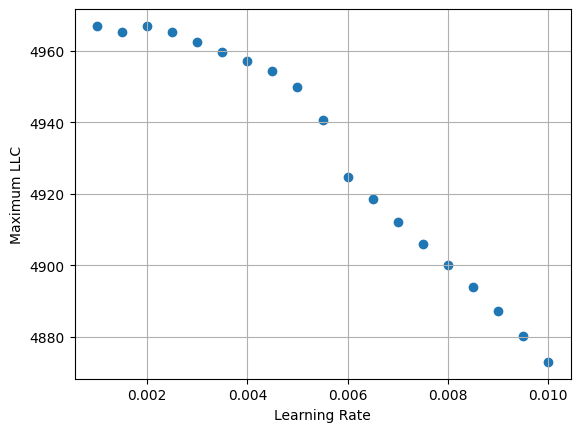}
        \caption{ Weight decay $=0$}
      \end{subfigure}\hfill
      \begin{subfigure}[t]{0.32\linewidth}
        \centering
        \includegraphics[width=\linewidth]{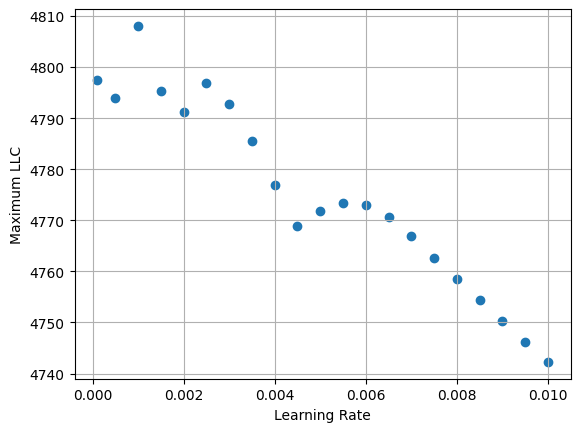}
        \caption{Weight decay $=0.0001$}
      \end{subfigure}\hfill
      \begin{subfigure}[t]{0.32\linewidth}
        \centering
        \includegraphics[width=\linewidth]{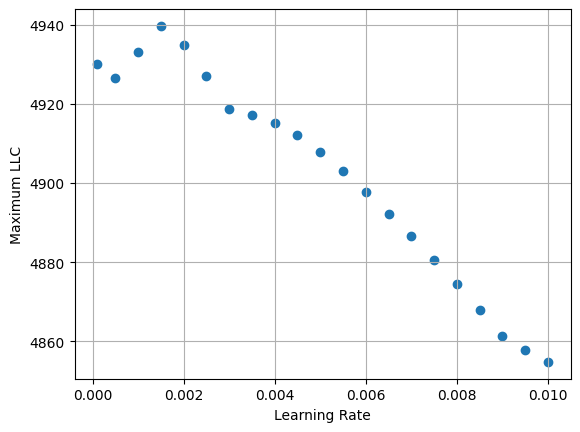}
        \caption{ Weight decay$=0.00001$}
      \end{subfigure}
      \caption{Relationship between the learning rate and the maximum LLC achieved during training. }
    \end{figure}

\section{Hutchinson trace as a proxy for average local curvature}\label{appen:hutchinson}

To complement the LLC analysis, we also track the \emph{Hutchinson trace estimator} of the Hessian trace \citep{hutchinson1990stochastic}. Let
\[
H(\theta) := \nabla_\theta^2 \mathcal{L}(\theta) \in \mathbb{R}^{P \times P}
\]
denote the Hessian of the loss with respect to the model parameters. If $\mathcal{L}$ is twice continuously differentiable, then for any perturbation $\delta \in \mathbb{R}^P$ we have the second-order expansion
\[
\mathcal{L}(\theta+\delta)
=
\mathcal{L}(\theta)
+
\nabla_\theta \mathcal{L}(\theta)^\top \delta
+
\frac{1}{2}\delta^\top H(\theta)\delta
+
o(\|\delta\|^2).
\]
Hence, near a critical point $\theta^\star$ with $\nabla_\theta \mathcal{L}(\theta^\star)\approx 0$, the local geometry of the loss is governed to second order by $H(\theta^\star)$.

Let $\lambda_1,\dots,\lambda_P$ denote the eigenvalues of $H(\theta^\star)$. If $\theta^\star$ is a local minimum and $H(\theta^\star)\succeq 0$, then
\[
\Tr(H(\theta^\star)) = \sum_{i=1}^P \lambda_i \geq 0,
\]
so the trace measures the \emph{total local curvature} of the basin. Equivalently, if $u$ is uniformly distributed on the unit sphere $\mathbb{S}^{P-1}$, then
\[
\mathbb{E}\big[u^\top H(\theta^\star)u\big] = \frac{1}{P}\Tr(H(\theta^\star)),
\]
and if $\varepsilon \sim \mathcal{N}(0,\sigma^2 I_P)$, then
\[
\mathbb{E}\big[\varepsilon^\top H(\theta^\star)\varepsilon\big] = \sigma^2 \Tr(H(\theta^\star)).
\]
Therefore, in the near-minimum regime, $\Tr(H)$ can be interpreted as the \emph{average curvature under isotropic perturbations}. In particular, smaller values of $\Tr(H)$ correspond to flatter local basins on average.

We stress, however, that $\Tr(H)$ is not a complete notion of flatness. Away from a local minimum it is only a \emph{signed} average curvature, so positive and negative eigenvalues may cancel. For example, the function $f(x,y)=x^2-y^2$ has Hessian
\[
\nabla^2 f(x,y)=
\begin{pmatrix}
2 & 0\\
0 & -2
\end{pmatrix},
\qquad
\Tr(\nabla^2 f)=0,
\]
despite being highly non-flat. Thus, throughout this appendix, we interpret the Hessian trace as a proxy for average local curvature, rather than as a universal notion of flatness.

The trace can be estimated efficiently without forming the Hessian explicitly. Let $v\in\mathbb{R}^P$ be a random vector satisfying
\[
\mathbb{E}[vv^\top]=I_P,
\]
for example a standard Gaussian vector or a Rademacher vector with independent coordinates taking values $\pm1$ with equal probability. Then
\[
\mathbb{E}[v^\top H v]
=
\mathbb{E}[\Tr(v^\top H v)]
=
\mathbb{E}[\Tr(Hvv^\top)]
=
\Tr\!\big(H\,\mathbb{E}[vv^\top]\big)
=
\Tr(H).
\]
Hence,
\[
\widehat{\Tr}(H)
:=
\frac{1}{m}\sum_{k=1}^m v_k^\top H v_k
\]
is an unbiased Monte Carlo estimator of $\Tr(H)$. In our experiments, we utilise $m = 100$ random vectors. 

Crucially, this only requires Hessian--vector products. Indeed, if $g(\theta)=\nabla_\theta \mathcal{L}(\theta)$, then
\[
H(\theta)v
=
\nabla_\theta\!\big(g(\theta)^\top v\big),
\]
so $Hv$ can be computed using automatic differentiation without explicitly materialising $H$ \citep{pearlmutter1994fast}. This makes the Hutchinson estimator practical for tracking curvature during training, even when repeated many times.

In summary, the Hutchinson trace should be viewed in our experiments as a computationally efficient estimator of the \emph{average local curvature} of the loss landscape. In the near-minimum regime relevant to grokking, decreases in $\Tr(H)$ indicate that optimisation is moving toward broader, less curved basins, which is precisely the qualitative phenomenon that the LLC is intended to capture.

\newpage


\end{document}